\documentclass{article}

\PassOptionsToPackage{numbers, compress}{natbib}
\usepackage{booktabs} 
\usepackage{xcolor}
\newcommand\ie{\textit{i.e.}}

\usepackage[preprint]{neurips_2026}
\usepackage{enumitem}
\setlist{leftmargin=15pt} 
\setlist[itemize]{noitemsep, topsep=0pt}

\usepackage{amsmath}
\usepackage{amssymb}
\usepackage{mathtools}
\usepackage{amsthm}
\usepackage{hyperref}
\usepackage[capitalize,noabbrev]{cleveref}
\usepackage{tabularx}

\usepackage{caption}    
\usepackage{subcaption} 
\usepackage{array}
\usepackage{bbm}
\usepackage{amsmath,amsfonts,amssymb}
\usepackage{graphicx}
\usepackage{hyperref}
\usepackage{enumitem}
\usepackage{geometry}
\usepackage{amsthm}
\usepackage{thmtools,thm-restate, multirow, multicol,xspace}
\usepackage{algorithm, algorithmic}
\usepackage{xcolor}

\newcommand{\bs}[1]{\boldsymbol{#1}}

\newtheorem{theorem}{Theorem}[section]
\newtheorem{proposition}[theorem]{Proposition}

\newtheorem{definition}[theorem]{Definition}
\newtheorem{assumption}[theorem]{Assumption}

\usepackage[most]{tcolorbox}
\usepackage{listings}
\usepackage{xcolor}

\definecolor{lightgraybg}{RGB}{248,248,248}
\definecolor{framegray}{RGB}{220,220,220}
\definecolor{darkblue}{RGB}{40,70,140}
\definecolor{darkgreen}{RGB}{20,110,60}

\lstdefinestyle{promptstyle}{
  basicstyle=\ttfamily\footnotesize,
  backgroundcolor=\color{lightgraybg},
  frame=none,
  breaklines=true,
  columns=fullflexible,
  keepspaces=true,
  showstringspaces=false,
  aboveskip=0pt,
  belowskip=0pt
}

\newtcblisting{promptbox}[2][]{
  colback=white,
  colframe=framegray,
  boxrule=0.5pt,
  arc=2pt,
  left=6pt,
  right=6pt,
  top=6pt,
  bottom=6pt,
  title={#2},
  fonttitle=\bfseries,
  coltitle=black,
  listing only,
  listing options={style=promptstyle},
  #1
}

\title{Remember with Confidence: Uncertainty Quantification for Spatio-temporal Memory with Probabilistic Guarantees}

\author{%
  Harry Zhang\thanks{Equal contribution.} \\
  MIT\\
  \texttt{harryz@mit.edu} \\
  \And
  Nicolas Gorlo\footnotemark[1] \\
  MIT\\
  \texttt{ngorlo@mit.edu} \\
  \And
  Luca Carlone \\
  MIT\\
  \texttt{lcarlone@mit.edu} \\
  }

\makeatletter
\def\@gobbletrailingpunct{%
  \@ifnextchar.{\@gobble\@gobbletrailingpunct}{%
    \@ifnextchar,{\@gobble\@gobbletrailingpunct}{%
      \@ifnextchar\space{\@gobble\@gobbletrailingpunct}{}%
    }%
  }%
}
\newcommand{\linkToPdf}[1]{\@gobbletrailingpunct}
\newcommand{\linkToPpt}[1]{\@gobbletrailingpunct}
\newcommand{\linkToCode}[1]{\@gobbletrailingpunct}
\newcommand{\linkToWeb}[1]{\@gobbletrailingpunct}
\newcommand{\linkToVideo}[1]{\@gobbletrailingpunct}
\newcommand{\linkToMedia}[1]{\@gobbletrailingpunct}
\newcommand{\award}[1]{\@gobbletrailingpunct}
\makeatother

\begin{document}
\maketitle

\begin{abstract}
Long-horizon robot operation requires spatio-temporal memory to record the environment state and recall it for downstream reasoning.
Scene graphs and retrieval-augmented systems ground VLM descriptions to persistent 3D entities with rich semantic descriptions.
However, VLM captions are noisy and viewpoint-inconsistent, and existing systems treat them as an oracle with no mechanism to detect unreliable stored descriptions.
We introduce object-level semantic uncertainty for multi-view VLM memory: a score that measures object-centric cross-view semantic scatter of captions and identifies semantically unresolved objects.
Then, we include our uncertainty scores in an advanced spatial-semantic memory system, that we dub UQ-DAAAM. UQ-DAAAM uses this score to actively refine uncertain objects under a fixed query budget by selecting high-quality views and fusing the resulting multi-view captions into a single object description.
We also derive probabilistic guarantees showing that higher-quality candidate views (as selected by our approach) are more likely to reduce uncertainty.
Our experiments show that uncertainty quantification can make embodied 4D memory systems more reliable and more effective.
In particular, on the OC-NaVQA benchmark, UQ-DAAAM achieves substantially larger uncertainty reduction and better spatio-temporal question answering performance than baselines. 
\end{abstract}    
\section{Introduction}

Safe long-horizon robot operation benefits from a persistent semantic memory of objects, agents, and regions encountered over time, which can be queried later for downstream reasoning and decision-making.
Recent work has shown that 3D scene graphs~\citep{Gorlo26cvpr-DAAAM,Hughes24ijrr-hydraFoundations,Gu24icra-conceptgraphs,Werby24rss-hovsg,Koch24cvpr-Open3DSG} can serve this role by grounding rich natural-language descriptions to persistent 3D entities observed by a robot.
A parallel line of work~\citep{Anwar25icra-remembr,Xie24arxiv-EmbodiedRAG,Yang25cvpr-3dmem} skips per-object 3D grounding and instead stores per-frame observations or captions in a retrieval database.
Storing articulate captions has the advantage that the memory enabling the robot to reason about its spatio-temporal environment is directly interpretable by a human operator. 
In practice, semantic annotations of these systems come from vision-language models (VLMs), which produce open-vocabulary descriptions from camera images and possibly object masks. 
In the robot perception loop, VLMs convert raw visual observations into descriptions retrievable at query time that support downstream
applications such as spatio-temporal question answering~\citep{Gorlo26cvpr-DAAAM,Anwar25icra-remembr} and task grounding~\cite{Zhang24arxiv-taskOrientedGrounding,Maggio24ral-clio,Gorlo26cvpr-DAAAM}.

However, VLM outputs are often noisy and semantically inconsistent across viewpoints~\citep{Bai24arxiv-hallucination, Chen24arxiv-inside}.
An object observed far away, from an occluded angle, or under difficult lighting conditions
can result in a caption that confidently states incorrect attributes of the object.
Existing memory systems treat VLM annotations as an oracle~\citep{Gu24icra-conceptgraphs, Anwar25icra-remembr, Yang25cvpr-3dmem},
with no mechanism to detect if a stored description is unreliable.
For instance, in a memory system like DAAAM~\citep{Gorlo26cvpr-DAAAM}, which assigns each object
a small fixed set of representative views and stores their captions in its 4D scene graph representation:
if those views happen to be poor, the resulting memory entry is silently incorrect.

We address this gap with \textbf{UQ-DAAAM}, an uncertainty-aware extension of DAAAM that
gives each object in DAAAM's 4D scene graph an explicit semantic uncertainty score,
and uses it for targeted refinement under a fixed budget of additional VLM queries.
Our uncertainty metric measures \emph{cross-view semantic scatter}, the log-volume spanned by the per-view caption 
embeddings of an object, which represents how much the VLM's descriptions disagree across views, following
the volumetric semantic measurement framework of~\citet{Lau25icmlws-uqMllm} and~\citet{Zhang26arxiv-fuse}.
This score identifies objects which descriptions are semantically unresolved and cannot yet be
trusted for downstream reasoning. UQ-DAAAM then performs \emph{active refinement}:
it allocates the extra query budget preferentially to uncertain objects,
selects additional views using DAAAM's view-quality heuristic \cite{Gorlo26cvpr-DAAAM}, which scores each (frame, object) pair 
by visibility and pixel coverage, and fuses the resulting
multi-view descriptions into a consolidated object description with reduced uncertainty.

We provide a probabilistic justification for the view selection policy and show that, under the assumption of 
first-order stochastic dominance, views with higher quality scores are more likely to reduce object uncertainty, 
giving a principled basis for using the existing quality heuristic to guide refinement.
We also derive a closed-form expression for the one-step uncertainty change when adding a new view,
which enables efficient online computation without re-solving DAAAM's frame-assignment step.

We evaluate UQ-DAAAM on the OC-NaVQA benchmark~\citep{Gorlo26cvpr-DAAAM} for large-scale spatio-temporal
question answering. UQ-DAAAM substantially outperforms other baselines on
uncertainty reduction rate and downstream task performance,
while keeping the underlying scene graph infrastructure and VLM unchanged.
In summary, we contribute (i) an object-level uncertainty score over multi-view captions, 
(ii) a budgeted active-refinement and fusion strategy, 
(iii) a probabilistic guarantee linking DAAAM's view-quality heuristic to one-step uncertainty reduction with a closed-form update, and 
(iv) an integrated system, UQ-DAAAM, which achieves state-of-the-art results on OC-NaVQA.
\section{Related Work}

\paragraph{Spatio-temporal robot memory.}
3D scene graphs~\citep{Armeni19iccv-3DsceneGraphs, Rosinol20rss-dynamicSceneGraphs, Hughes22rss-hydra, Hughes24ijrr-hydraFoundations, Gorlo26cvpr-DAAAM} and retrieval-augmented memory systems~\citep{Anwar25icra-remembr, Xie24arxiv-EmbodiedRAG, Yang25cvpr-3dmem} provide structured representations for long-horizon robot operation, with open-vocabulary semantics attached via VLM features~\citep{Jatavallabhula23rss-ConceptFusion, Werby24rss-hovsg, Koch24cvpr-Open3DSG} or per-object VLM captions~\citep{Gu24icra-conceptgraphs}.
DAAAM~\citep{Gorlo26cvpr-DAAAM} combines both, by annotating a 4D scene-graph with localized VLM captions from a sparse set of representative view.
Across all these systems, VLM outputs are treated as oracle annotations~\citep{Gu24icra-conceptgraphs, Anwar25icra-remembr, Saxena25corl-GraphEQA, Ginting25corl-mindpalace, Honerkamp24ral-MoMa}; the few exceptions apply confidence only at the task-output level~\citep{Ren23corl-knowno, Ren24rss-exploreEQA}, not per-node memory.
We instead quantify and reduce semantic uncertainty at the object level, enabling the system to detect and resolve unreliable memory entries before they propagate to downstream reasoning.

\paragraph{Uncertainty quantification.} 
Classical uncertainty quantification (UQ) methods such as MC Dropout~\citep{Gal16icml-mcdropout}, Deep Ensembles~\citep{Lakshminarayanan17neurips-deepEnsembles}, and conformal prediction~\citep{Angelopoulos21arxiv-gentle, Barber23aos-conformal} address epistemic or aleatoric uncertainty but do not directly apply to multimodal large language models (MLLMs), where token-level confidence is poorly aligned with semantic correctness~\citep{Chou25aclf-mmR3, Lau25icmlws-uqMllm, Chen24arxiv-inside}.
Recent MLLM uncertainty work falls into three groups: input-level approaches that model uncertainty in multimodal embeddings~\citep{Tran22arxiv-plex, Upadhyay23iccv-probvlm, Venkataramanan25arxiv-probEmb}; output-level approaches that estimate uncertainty from response inconsistency via entropy, self-consistency, or embedding-space dispersion~\citep{Nikitin24neurips-kernelEntropy, Farquhar24nature-semanticEntropy, Manakul23arxiv-selfCheckGPT, Lau25icmlws-uqMllm, Chen24arxiv-inside}; and methods based on conformal prediction or external verification~\citep{Quach23arxiv-conformalLM, Su24arxiv-apiEnough, Khan24cvpr-consistency}.
Orthogonally, hallucination mitigation methods improve training data or model architectures~\citep{Liu23arxiv-aligningRobust, Yu24cvpr-halluciDoctor, Tong24cvpr-eyesWideShut, Zhai23arxiv-halleSwitch}, but residual uncertainty remains pervasive~\citep{Bai24arxiv-hallucination}.
Our uncertainty is represented at the object level and captures semantic disagreement across views, following the theoretical framework of UMPIRE~\cite{Lau25icmlws-uqMllm} and FUSE~\cite{Zhang26arxiv-fuse}.
\section{Problem Formulation}
\label{sec:problem}

We study \emph{semantic memory construction} for embodied agents.
Given an RGB-D stream $\mathcal{F}=\{f_1,\dots,f_n\}$, an object-centric perception system~\cite{Gorlo26cvpr-DAAAM} yields object observations (or \emph{fragments}) $\mathcal{O}=\{o_1,\dots,o_m\}$ tracked across frames.
The system constructs a structured semantic memory grounded in a geometric reconstruction of the spatio-temporal environment by associating each $o_j$ with natural-language descriptions from a selected view set $\mathcal{V}_j$;
since descriptions can vary across views, the initial memory may leave some objects semantically ambiguous.
We define an object-level uncertainty score $u_j = \mathcal{U}(o_j, \mathcal{V}_j)$ to quantify this ambiguity. 
Under a fixed additional query budget, \emph{active refinement} then acquires further views for uncertain objects, allocating queries to the fragments and viewpoints most likely to resolve remaining ambiguity for downstream reasoning.
\section{Method}
\begin{figure}
\centering
\includegraphics[width=\textwidth]{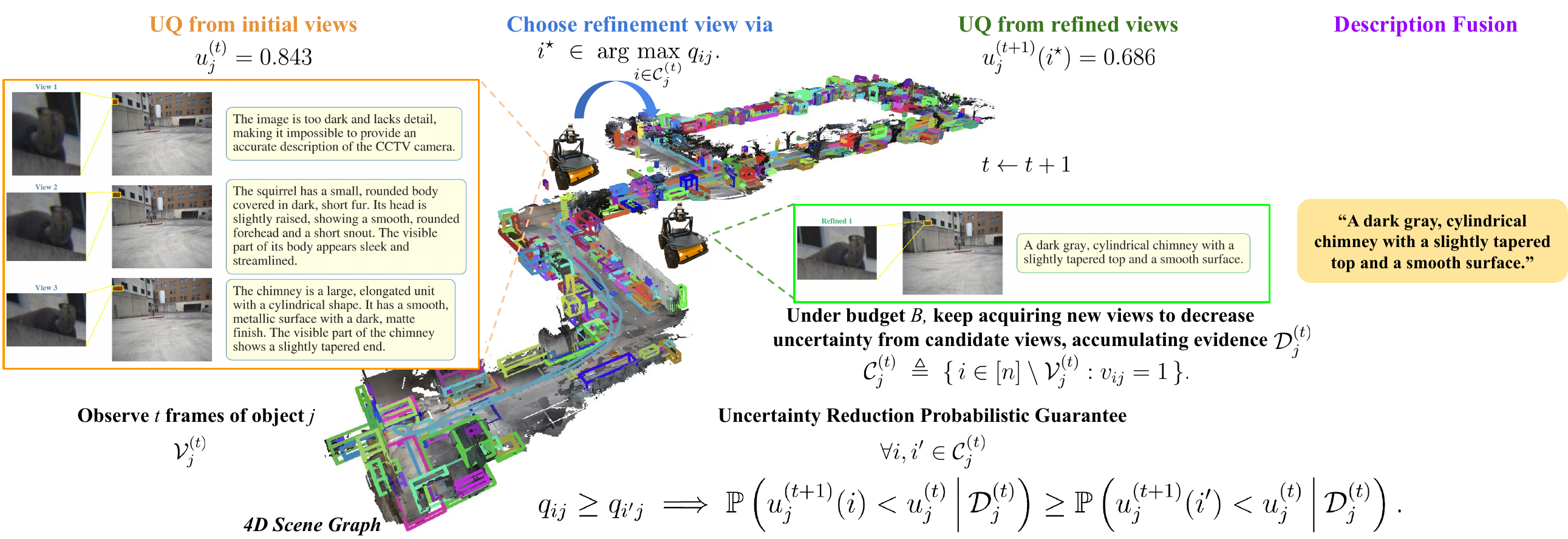}
\caption{Overview of UQ-DAAAM. Starting from DAAAM's 4D scene graph, we first compute an object-level 
uncertainty score from the captions associated with the initial selected views. If an object is uncertain, 
additional refinement views are chosen from the feasible set using the view-quality heuristic $q_{ij}$, supported by our probabilistic guarantee. 
Refinement continues under a fixed budget, reducing the uncertainty. Captions from all views are fused into a single description.}
\end{figure}

\paragraph{Background.}
We build upon DAAAM~\citep{Gorlo26cvpr-DAAAM}, which constructs 4D scene graphs in real time from an RGB-D stream.
DAAAM tracks object fragments $\{o_j\}_{j=1}^m$ across a rolling window of candidate frames $\{f_i\}_{i=1}^n$, 
maintaining a visibility matrix $v_{ij}\in\{0,1\}$ and a heuristic view-quality score $q_{ij}\in[0,1]$ per (frame, fragment) pair.
In DAAAM \cite{Gorlo26cvpr-DAAAM}, $q_{ij}$ is large when object $j$ is centered and large in frame $i$, and vice versa.
A binary linear program (BLP) assigns each fragment to \emph{a single} selected frame via a binary decision variable $y_{ij}$, maximizing $\sum_{i,j} q_{ij} y_{ij}$ subject to visibility and frame-budget constraints.
The selected (frame, fragment) pairs are captioned by the localized captioning model, DAM~\citep{Lian25iccv-DAM}. A mapping backend builds a scene graph from the RGB-D sensor data and the annotated fragments. Each object forms a node in the scene graph. The captions then form a per-node description history.

\paragraph{From single-view to multi-view assignment.}
For object-level uncertainty quantification, we need a measure of view-dependent disagreement. A single caption per fragment provides no signal in that sense. We therefore relax the BLP so that each fragment is covered by multiple views rather than exactly one. Letting $y_{ij}\in\{0,1\}$ indicate whether frame $f_i$ is one of the representative views for $o_j$, we replace the single-assignment constraint with
$
    \sum_{i=1}^{n} y_{ij} = r \;\; \forall j\in[m],
$
where $r$ is a fixed per-object frame budget (we use $r=3$), while retaining $y_{ij}\le v_{ij}$ (only select view $f_i$ if $o_j$ is visible) and $y_{ij}\le x_i$ (only consider views under budget). The BLP objective function $\sum_{i,j} q_{ij}y_{ij}$ is unchanged, optimizing for high-quality views. For each object $o_j$, we run DAM on every view in $\mathcal{V}_j=\{f_i: y_{ij}=1\}$, producing the multi-view caption set $\{c_{ij}\}_{i\in\mathcal{V}_j}$ used by the rest of the method.




\subsection{Uncertainty Quantification}
\label{sec:uq}

To quantify object-centric semantic uncertainty, we adapt the idea of volumetric semantic measurement \cite{Zhang26arxiv-fuse,Lau25icmlws-uqMllm} to the multi-view robot memory construction setting of DAAAM. 
In our setting, each object is observed from multiple selected views, and each view yields a natural-language description produced by a VLM. 
Uncertainty should increase both when these view-conditioned descriptions are semantically inconsistent and when individual descriptions are internally less coherent. 
Our formulation therefore combines two signals: \emph{cross-view semantic scatter} and \emph{per-view incoherence}.

\paragraph{Multi-view semantic evidence for an object.}

Consider an object fragment $o_j$. Suppose DAAAM selects a set of representative views $\mathcal{V}_j = \{f_{i_1}, \dots, f_{i_r}\}$
for this object, where $r$ is the number of selected views. For each selected view $f_{i}\in \mathcal{V}_j$, 
we query the VLM on the localized object crop and obtain a caption $c_{ij}$.
We then embed each caption into a semantic vector space using a text encoder $\phi(\cdot)$ and normalize the result:
$
\bs r_{ij}
\;\triangleq\;
\frac{\phi(c_{ij})}{\|\phi(c_{ij})\|_2}
\;\in\;
\mathbb{R}^d.
$
Stacking the normalized caption embeddings column-wise gives the object-level \textit{semantic evidence} matrix $
\bs R_j
\;\triangleq\;
[\bs r_{1j},\dots,\bs r_{rj}]
\;\in\;
\mathbb{R}^{d\times r}.
$
This matrix captures the semantic content of the object as perceived from the selected views.

\paragraph{Incoherence weighting.}
Following \cite{Zhang26arxiv-fuse,Lau25icmlws-uqMllm}, semantic diversity alone is insufficient for measuring uncertainty, 
since a set of captions could be similar but each incoherent.
We therefore assign each caption $c_{ij}$ an incoherence weight $\gamma_{ij} \triangleq \exp(\alpha(1-p_{ij}))$,
where $p_{ij}\in[0,1]$ is the average per-token probability and $\alpha\ge 0$ is a scaling hyperparameter.
A more coherent caption yields smaller $\gamma_{ij}$; a less coherent one inflates it.
The weights are collected in $\bs\Gamma_j \triangleq \mathrm{diag}(\gamma_{1j},\dots,\gamma_{rj})\in\mathbb{R}^{r\times r}$.

\paragraph{Uncertainty measure via semantic scatter.}

To combine cross-view semantics and per-view incoherence, we construct an incoherence-adjusted semantic evidence matrix
$
\bs Z_j
\;\triangleq\;
\bs R_j \bs \Gamma_j
\;\in\;
\mathbb{R}^{d\times r}.
$
Equivalently, the $\ell$-th column of $\bs Z_j$ is the weighted embedding
$
\bs z_{ij} = \gamma_{ij} \bs r_{ij}.
$
\begin{definition}[Semantic scatter and uncertainty score]
For an object fragment $o_j$, given incoherence-adjusted semantic evidence $\bs Z_j$, we define the semantic scatter matrix as
$
\bs S_j
\;\triangleq\;
\bs Z_j^\top \bs Z_j + \lambda \bs I_r
\;=\;
\bs \Gamma_j \bs R_j^\top \bs R_j \bs \Gamma_j + \lambda \bs I_r,
$
where $\lambda>0$ is a small regularization constant.
Our object-level uncertainty score is then defined as the log-determinant of 
the semantic scatter matrix
\begin{equation}
 u_j
\;\triangleq\;
\log\det \bs S_j.
\end{equation}
\end{definition}
The determinant of $\bs Z_j^\top \bs Z_j$ measures the squared volume spanned by $\bs Z_j$'s column vectors. 
$\det(\bs S_j)$, the semantic volume generated by the weighted multi-view caption embeddings, becomes large when (i) the captions 
from different views point in different semantic directions (increasing cross-view dispersion), or (ii) one or more captions are highly incoherent, inflating their weighted contribution through $\bs \Gamma_j$.
Consequently, our uncertainty score $u_j$ increases when the object's interpretation is unstable across views or unreliable within views.\footnote{We provide a more extensive rationale
behind object-level uncertainty in \cref{app:objectlevel}}.

\subsection{Active View Selection}
Assume a calibrated threshold value $\tau$: if $u_j$ exceeds $\tau$, then it implies that the agent is uncertain about 
the semantic interpretation of $o_j$. We let the agent actively acquire additional views for the same object fragment in an online manner: 
it acquires further views until the uncertainty falls below $\tau$ or a budget runs out. In the following discussion, we assume
that the agent has already selected $t$ views for object $o_j$ and is considering which additional view to acquire next. 
We denote $u_j^{(t)}$ as the uncertainty score computed from the current $t$ views. 
If $u_j^{(t)}>\tau$, then the agent is uncertain about $o_j$ based on the current evidence, and we wish to assign additional views 
to lower $u_j^{(t+1)}$.

\paragraph{Active view selection as BOED.} 
Let $\theta_j$ denote the latent semantic state of object $o_j$, and let
$\mathcal{V}_j^{(t)}\subseteq[n]$ denote the set of view indices selected for $o_j$
after $t$ acquisitions, so that
$\bs Z_j^{(t)}\in\mathbb{R}^{d\times t}$ stacks the weighted embeddings
$\{\bs z_{ij}\}_{i\in \mathcal{V}_j^{(t)}}$ and
$\bs S_j^{(t)} \triangleq (\bs Z_j^{(t)})^\top \bs Z_j^{(t)} + \lambda \bs I_t$.
We write $\mathcal{D}_j^{(t)}$ for the evidence accumulated after these $t$ views,
and define the feasible candidate views as the set of views where
$o_j$ is visible in the rolling window but have not yet been selected:
\begin{equation}
    \label{eq:cj-set}
\mathcal{C}_j^{(t)}
\;\triangleq\;
\{\, i\in [n]\setminus \mathcal{V}_j^{(t)} : v_{ij}=1 \,\}.
\end{equation}
In the language of Bayesian Optimal Experimental Design (BOED) \cite{Chaloner95ss-bayesianDOE}, each candidate view $f_i$ is an \emph{experiment}, 
$\theta_j$ is the \emph{latent parameter of interest}, 
and the future VLM output from view $i$ is the \emph{experimental outcome}. 
The exact BOED \emph{utility} of view $i$ is the expected reduction of uncertainty about $\theta_j$ that would 
result from acquiring view $i$ and observing its caption embedding. 
\begin{definition}[Exact one-step BOED utility and acquisition rule]
The one-step utility of candidate view $i$ is defined as the expected uncertainty reduction
$
\Delta_j(i \mid \mathcal{D}_j^{(t)})
\;\triangleq\;
\mathbb{E}\left[
u_j^{(t)} - u_j^{(t+1)}(i)
\,\middle|\,
\mathcal{D}_j^{(t)}
\right],
$
where $u_j^{(t+1)}(i)$ is the uncertainty score after acquiring view $i$.
The next refinement view should therefore be chosen by
\begin{equation}
\label{eq:argmax}
i_j^\star
\;\in\;
\arg\max_{i\in \mathcal{C}_j^{(t)}} \Delta_j(i \mid \mathcal{D}_j^{(t)}).
\end{equation}
\end{definition}

\paragraph{One-step update via augmented scatter matrix.}
Given $\bs r_{ij}$ as the normalized VLM output embedding when presented with view $i$ for object $j$, $\gamma_{ij}>0$ as its incoherence weight, 
and weighted candidate embedding $\bs z_{ij} \;\triangleq\; \gamma_{ij}\,\bs r_{ij}\in\mathbb{R}^d,$
we define the similarity vector
$
\bs b_{ij}^{(t)}
\;\triangleq\;
(\bs Z_j^{(t)})^\top \bs z_{ij}
\in \mathbb{R}^t.
$
Denoting the semantic scatter matrix obtained via $t$ currently collected views of object $j$ as $\bs S_j^{(t)}$, 
we can apply matrix algebra to write out the augmented scatter matrix after adding view $i$ as
\begin{equation}
    \label{eq:onestepupdate}
\bs S_j^{(t+1)}(i)
=
\begin{bmatrix}
\bs S_j^{(t)} & \bs b_{ij}^{(t)} \\
(\bs b_{ij}^{(t)})^\top & \gamma_{ij}^2+\lambda
\end{bmatrix}.
\end{equation}
With this formulation, we can derive a closed-form expression for the one-step uncertainty update.
For sake of clarity, we first define two intermediate random variables.
\begin{definition}[Redundancy score and Schur score]
    \label{def:redandschur}
    For any candidate view $i\in\mathcal{C}_j^{(t)}$ acquired for object $j$ with $t$ existing views, we define the redundancy score $\eta_{ij}^{(t)}$ and Schur score $s_{ij}^{(t)}$ as
    \begin{equation}
    \eta_{ij}^{(t)}
    \;\triangleq\;
    (\bs b_{ij}^{(t)})^\top (\bs S_j^{(t)})^{-1} \bs b_{ij}^{(t)},\quad
        s_{ij}^{(t)}
        \;\triangleq\;
        \gamma_{ij}^2+\lambda-\eta_{ij}^{(t)}.
    \end{equation}
\end{definition}
The redundancy score $\eta_{ij}^{(t)}$ measures how much newly acquired information is already known and 
is large when the new weighted caption embedding is well explained by the currently selected views. The Schur score $s_{ij}^{(t)}$ is 
the Schur complement of $\bs S_j^{(t)}$ in $\bs S_j^{(t+1)}(i)$ and 
captures the amount of new information that view $i$ would add after accounting for redundancy.
\begin{restatable}[Closed-form one-step uncertainty update]{theorem}{exactonestep}
\label{thm:exact-one-step}
For any candidate view $i\in\mathcal{C}_j^{(t)}$ acquired for object $j$ with $t$ existing views, the one-step uncertainty update is given by
\begin{equation}
u_j^{(t+1)}(i)-u_j^{(t)}
=
\log\Bigl(
s_{ij}^{(t)}
\Bigr).
\end{equation}
\end{restatable}


\begin{restatable}[Criterion for one-step uncertainty reduction]{theorem}{exactcriterion}
\label{thm:exactcriterion}
For any candidate view $i\in\mathcal{C}_j^{(t)}$ acquired for object $j$ with $t$ existing views, we have the following uncertainty reduction criterion:
\[
u_j^{(t+1)}(i) < u_j^{(t)}
\quad\Longleftrightarrow\quad
\eta_{ij}^{(t)} > \gamma_{ij}^2+\lambda-1.
\]
\end{restatable}

Intuitively, to reduce uncertainty, the redundancy (agreement with existing evidence) $\eta_{ij}$ of a new view needs to overcome both regularization and its own incoherence.

\paragraph{BOED-inspired proxy selection using the view-quality score.}
So far, our analysis suggests that for a candidate view $i\in\mathcal{C}_j^{(t)}$, the exact one-step uncertainty reduction is
$u_j^{(t)} - u_j^{(t+1)}(i)=-\log s_{ij}^{(t)},$ and the exact BOED utility is thus
$\Delta_j(i)\;\triangleq\;\mathbb{E}\left[-\log s_{ij}^{(t)}\,\middle|\,\mathcal{D}_j^{(t)}\right].$
However, evaluating $\Delta_j(i)$ exactly would require querying the VLM on every feasible candidate view, since $s_{ij}^{(t)}$ depends on the future caption embedding and future incoherence weight.

We avoid this cost by using the precomputed view-quality score $q_{ij}$ already used in DAAAM as a cheap surrogate for expected informativeness. 
For a fixed object $o_j$, we expect higher-quality views to result in smaller Schur scores, and thus larger one-step uncertainty reductions.
Towards this end, we make the following 
assumption on the relationship between $q_{ij}$ and $s_{ij}^{(t)}$.

\begin{assumption}[First-order stochastic dominance]
\label{assump:fosd}
For each object $o_j$, the conditional distribution of the random Schur score $s_{ij}^{(t)}$ is stochastically nonincreasing in $q_{ij}$. That is, for any two candidates $i,i'\in\mathcal{C}_j^{(t)}$,
\[
q_{ij}\ge q_{i'j}
\quad\Longrightarrow\quad
\mathbb{P}\left(s_{ij}^{(t)} \le x \,\middle|\, \mathcal{D}_j^{(t)}\right)
\;\ge\;
\mathbb{P}\left(s_{i'j}^{(t)} \le x \,\middle|\, \mathcal{D}_j^{(t)}\right)
\qquad
\forall x>0.
\]
Equivalently, conditional on the current evidence, a higher-quality view produces a Schur score that is smaller in the sense of first-order stochastic dominance (FOSD).
\end{assumption}

This assumption is natural because the Schur score $s_{ij}^{(t)}$
decreases when the future caption is less incoherent (smaller $\gamma_{ij}$) and more explained by the current evidence (larger $\eta_{ij}^{(t)}$). Thus, a higher-quality view is assumed to shift the distribution of $s_{ij}^{(t)}$ toward smaller values.

\begin{restatable}[Monotonicity of BOED utility under FOSD and proxy acquisition]{proposition}{fosdm}
\label{prop:fosd-mono}
Suppose the FOSD assumption holds for object $o_j$. Then for any two candidates $i,i'\in\mathcal{C}_j^{(t)}$,
$
q_{ij}\ge q_{i'j}
\quad\Longrightarrow\quad
\Delta_j(i)\ge \Delta_j(i').
$
Consequently, any maximizer of $q_{ij}$ over $\mathcal{C}_j^{(t)}$ is also a maximizer of the exact BOED utility $\Delta_j(i)$ and candidate
 views can be ranked by $q_{ij}$, yielding the proxy acquisition rule:
 \begin{equation}
\label{eq:rule}
i_j^\star
\;\in\;
\arg\max_{i\in\mathcal{C}_j^{(t)}} q_{ij}.
\end{equation}
\end{restatable}

\begin{restatable}[Probabilistic guarantee of uncertainty reduction via heuristic]{theorem}{fosdprob}
\label{prop:fosd-prop}
Suppose that for two feasible candidate views $i,i'\in\mathcal C_j^{(t)}$, $q_{ij}\ge q_{i'j}$
and the corresponding Schur scores $s_{ij}^{(t)}, s_{i'j}^{(t)}$ satisfies \cref{assump:fosd}.
Then
\[
\mathbb{P}\left(u_j^{(t+1)}(i)<u_j^{(t)} \,\middle|\, \mathcal D_j^{(t)}\right)
\;\ge\;
\mathbb{P}\left(u_j^{(t+1)}(i')<u_j^{(t)} \,\middle|\, \mathcal D_j^{(t)}\right).
\]
\end{restatable}
This gives us a probabilistic guarantee for using \cref{eq:rule} as the rule to select more frames via heuristic score $q_{ij}$ to reduce uncertainty.

\paragraph{Overall active view refinement procedure.}
Backed by our theoretical analysis, for each object $o_j$ with $u_j>\tau$, we allocate a per-object refinement budget $B$. We collect feasible additional views $\mathcal{C}_j^{(t)}$ for 
$o_j$ via \cref{eq:cj-set}. Next, we rank the candidates in $\mathcal{C}_j^{(t)}$ by the 
heuristic view-quality score $q_{ij}$ and select the top-$B$ views. The VLM is then queried on these additional views sequentially, 
and each resulting caption is incorporated into the object's multi-view semantic representation. 
After each update, the object-level uncertainty is recomputed. Refinement stops when the uncertainty falls below $\tau$, 
the budget $B$ is exhausted, or no feasible candidate views remain. Note that our refinement procedure can be understood as an ``overt attention'' process to reduce uncertainty.

\subsection{Multi-view Descriptions Aggregation}
After uncertainty-driven view selection, each object is associated with a small set of captions obtained from multiple views.
We fuse these captions into a single object description using a constrained LLM-based aggregation step.
Captions are ordered by acquisition stage (initial views first, refinement views last) and the LLM is instructed to privilege 
later (refinement) captions when resolving conflicts. When disagreement exists only among initial views, the fused caption preserves 
the alternatives using cautious language (e.g., \emph{``white or gray''}) rather than collapsing ambiguity prematurely. 
The LLM is further constrained to use only information present in at least one source caption, preventing hallucinated facts.
The full prompt specification and output format are given in \cref{app:llm-fusion}; the overall UQ-DAAAM procedure is summarized in \cref{app:alg}.
\section{Experiments}
\label{sec:experiments}

In our implementation, we set $r=3$ to gather diverse perspectives on each object to measure the semantic scatter as the uncertainty. 
We use a budget of $B=2$ for refinement, and the uncertainty threshold $\tau$ is set to the top 20-th percentile of the object uncertainty distribution, 
selected on a held-out calibration dataset, from which we also derive the statistics to normalize the uncertainty scores. In the calculation
of $u$, we use a small regularization constant of $\lambda=1e-8$ to ensure numerical stability of the log-determinant. We set the temperature for caption generation to 1. 
For caption fusion, we use GPT-5-mini. Visualizations and analysis of qualitative results are in \cref{app:qual}.

\subsection{Budgeted Refinement on OC-NaVQA}
\label{sec:exp_ocnavqa}
\begin{table}[t]
\centering
\resizebox{\linewidth}{!}{%

\begin{tabular}{lcccccccccccc}
\toprule
 & \multicolumn{4}{c}{Descriptive Question Accuracy $\uparrow$} & \multicolumn{4}{c}{Position Error [m] $\downarrow$} & \multicolumn{4}{c}{Temporal Error [min] $\downarrow$} \\
\cmidrule(lr){2-5} \cmidrule(lr){6-9} \cmidrule(lr){10-13}
Method & Short & Medium & Long & All & Short & Medium & Long & All & Short & Medium & Long & All \\
\midrule
{DAAAM} & 0.679 & 0.753 & 0.690 & 0.711 & 22.07 & 41.69 & 53.47 & 41.75 & 3.669 & 0.914 & 1.438 & 1.792 \\
{DAAAM-RR} & 0.683 & 0.751 & 0.694 & 0.716 & 21.52 & 41.19 & 53.34 & 41.78 & 3.748 & 0.918 & 1.428 & 1.769 \\
{DAAAM-QR} & 0.715 & 0.760 & 0.748 & 0.738 & 17.83 & 29.47 & 57.26 & 40.14 & 2.284 & 0.905 & 1.408 & 1.734 \\
{DAAAM-NC}~\cite{Khan24cvpr-consistency} & 0.681 & 0.755 & 0.692 & 0.713 & 21.88 & 41.53 & 53.39 & 41.62 & 3.648 & 0.912 & 1.434 & 1.783 \\
{DAAAM-LN}~\cite{Malinin20arxiv-uncertaintyAR} & 0.682 & 0.756 & 0.694 & 0.714 & 21.74 & 41.45 & \textbf{53.31} & 41.55 & 3.623 & 0.910 & 1.430 & 1.776 \\
{DAAAM-Ent}~\cite{Kuhn23arxiv-semanticUnc} & 0.703 & 0.763 & 0.734 & 0.733 & 20.13 & 36.82 & 55.29 & 40.33 & 2.891 & 0.907 & 1.368 & 1.718 \\
{DAAAM-ES}~\cite{Chen24arxiv-inside} & 0.714 & 0.769 & 0.748 & 0.743 & 19.28 & 34.61 & 55.87 & 39.69 & 2.703 & \textbf{0.903} & 1.341 & 1.683 \\
{UQ-DAAAM} (Ours) & \textbf{0.731} & \textbf{0.778} & \textbf{0.786} & \textbf{0.761} & \textbf{15.46} & \textbf{25.35} & 58.94 & \textbf{37.84} & \textbf{1.948} & {0.910} & \textbf{1.275} & \textbf{1.589} \\
\bottomrule
\end{tabular}
}
\caption{Results on OC-NaVQA benchmark~\cite{Gorlo26cvpr-DAAAM,Anwar25icra-remembr}. All models use GPT-5-mini.}
\label{tab:navqa_full}
\end{table}

We evaluate UQ-DAAAM on spatio-temporal question answering (SQA) using the CODa dataset and the object-centric NaVQA (OC-NaVQA) 
benchmark introduced by DAAAM~\cite{Gorlo26cvpr-DAAAM}. DAAAM motivates OC-NaVQA as a large-scale, long-horizon extension of NaVQA~\cite{Anwar25icra-remembr}: 
it re-annotates spatial queries with actual object positions, improves label quality with a custom 3D labeling tool, 
and evaluates over the full sequence context rather than a restricted observation window, 
resulting in settings of up to 35.8 minutes and 1.64 km of traveled distance. 

Following DAAAM, we use the same task categories and report the same three metrics: 
\emph{question accuracy}, \emph{positional error}, and \emph{temporal error}. 
DAAAM uses these metrics both on the original NaVQA benchmark and on OC-NaVQA, 
where all methods are paired with the same reasoning model over the constructed memory. 
We adopt the same evaluation protocol so that improvements can be attributed to the memory 
representation and refinement strategy rather than changes in the reasoning backend. We report the main comparison at a fixed refinement budget in \cref{tab:navqa_full}. 
Specifically, we compare: \textbf{DAAAM}: the original pipeline with no post hoc refinement; \textbf{DAAAM-RR}: after the initial DAAAM pass, additional feasible object-view queries are allocated uniformly
 at random on objects under the same extra budget; \textbf{DAAAM-QR}: after the initial DAAAM pass, the extra budget is allocated only for low-quality objects via DAAAM's original 
view-quality score $q_{ij}$, without uncertainty; \textbf{UQ-DAAAM}: our full method, which uses object-level uncertainty to prioritize refinement and then updates memory 
with the additional multi-view semantic evidence.
In addition, we compare against 4 different UQ methods implemented on top of DAAAM: \textbf{DAAAM-NC}~\cite{Khan24cvpr-consistency}, which uses neighborhood consistency; 
\textbf{DAAAM-LN}~\cite{Malinin20arxiv-uncertaintyAR}, which uses next-token log likelihood as an uncertainty signal; 
\textbf{DAAAM-Ent}~\cite{Kuhn23arxiv-semanticUnc}, which uses entropy as an uncertainty signal; 
and \textbf{DAAAM-ES}~\cite{Chen24arxiv-inside}, which uses evidence score as an uncertainty signal. The details for the baselines 
can be found in \cref{app:uq-qual}.
As the results show in \cref{tab:navqa_full}, UQ-DAAAM outperforms the baselines across all three metrics, with the largest gains in question accuracy and positional error. 
This demonstrates that uncertainty-aware refinement can more effectively allocate a fixed semantic-query budget to improve the quality. 





\subsection{Uncertainty Reduction}
\label{sec:exp_uq_reduction}

In addition to downstream task performance, we directly evaluate whether refinement achieves its immediate objective: 
reducing object-level semantic uncertainty. For each object \(o_j\), we compute its uncertainty before refinement, 
\(u_j^{\mathrm{pre}}\), and after refinement, \(u_j^{\mathrm{post}}\). We focus in particular on the subset of 
objects that are initially marked as uncertain, i.e., $\mathcal{U}_{\tau}^{\mathrm{pre}}
=
\{\, j \in [m] : u_j^{\mathrm{pre}} > \tau \,\}.
$
We then measure how effectively each method reduces uncertainty within this set.

Our primary metric is the \emph{uncertainty reduction rate}, defined as the percentage of initially uncertain objects whose uncertainty falls below the threshold after refinement:
$
\mathrm{URR}
=
\frac{
\bigl|\{\, j \in \mathcal{U}_{\tau}^{\mathrm{pre}} : u_j^{\mathrm{post}} \le \tau \,\}\bigr|
}{
|\mathcal{U}_{\tau}^{\mathrm{pre}}|
}
$
This metric directly captures how often refinement succeeds in converting an uncertain object into a semantically resolved one. 
We compare the same DAAAM-RR, DAAAM-QR, and UQ-DAAAM variants from \cref{sec:exp_ocnavqa}; uncertainty-as-error-predictor results are deferred to \cref{app:uq-qual}.



The uncertainty-reduction results in \cref{fig:uq_reduction} show a clear advantage for UQ-DAAAM over both non-uncertainty-based refinement baselines.
At a fixed budget $B=2$, UQ-DAAAM achieves a URR of 92.3\%, compared to 34.2\% for quality-based refinement and only 5.6\% for random refinement.
UQ-DAAAM shifts the post-refinement distribution markedly to the left, showing a much larger and more consistent reduction in uncertainty across objects. 
Overall, these results indicate that the proposed uncertainty-aware refinement is substantially more effective.
\begin{figure}[t]
\begin{minipage}[t]{0.48\linewidth}
  \centering
  \includegraphics[width=\linewidth]{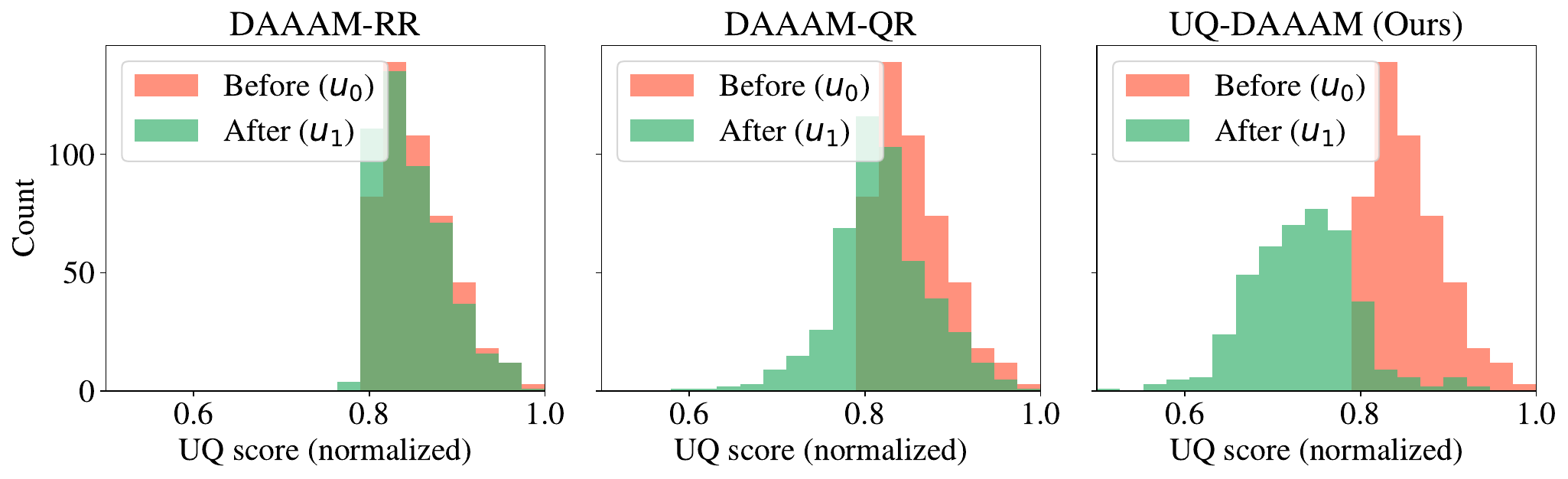}
  \vspace{0.5em}
  \resizebox{\linewidth}{!}{%
  \begin{tabular}{lccc}
  \toprule
   & DAAAM-RR & DAAAM-QR & UQ-DAAAM \\
  \midrule
  URR & 5.6\% & 34.2\% & 92.3\% \\
  Mean $\Delta_u$ & -0.006 & -0.036 & -0.121 \\
  \bottomrule
  \end{tabular}
  }
  \caption{Uncertainty reduction comparison across refinement strategies. 
  \emph{Top:} Uncertainty reduction distributions. 
  \emph{Bottom:} URR and mean uncertainty reduction at fixed budget $B=2$.}
  \label{fig:uq_reduction}
\end{minipage}
\hfill
\begin{minipage}[t]{0.48\linewidth}
  \centering
  \includegraphics[width=\linewidth]{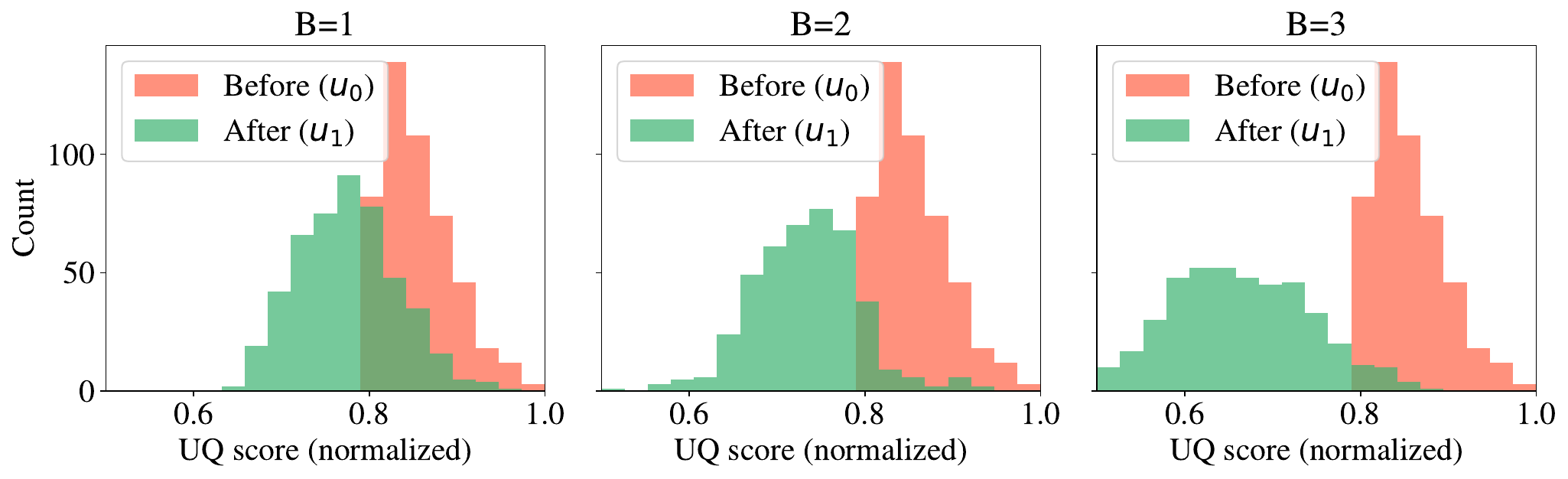}
  \vspace{0.5em}
  \resizebox{0.67\linewidth}{!}{%
  \begin{tabular}{lccc}
  \toprule
   & $B=1$ & $B=2$ & $B=3$ \\
  \midrule
  URR & 68.9\% & 92.3\% & 94.4\%\\
  Mean $\Delta_u$ &-0.079 & -0.121&-0.169 \\
  \bottomrule
  \end{tabular}
  }
  \caption{Uncertainty reduction comparison across budgets. 
  \emph{Top:} per-budget uncertainty reduction curves. 
  \emph{Bottom:} URR and mean uncertainty reduction.}
  \label{fig:uq_reduction_right}
\end{minipage}
\end{figure}

\subsection{Ablations}
\begin{table}[t]
\centering
\resizebox{\linewidth}{!}{%
\begin{tabular}{lcccccccccccccccccccc}
\toprule
 & \multicolumn{3}{c}{Budget} & \multicolumn{4}{c}{Object} & \multicolumn{3}{c}{View} & \multicolumn{5}{c}{Fusion} & \multicolumn{4}{c}{Threshold} \\
\cmidrule(lr){2-4} \cmidrule(lr){5-8} \cmidrule(lr){9-11} \cmidrule(lr){12-16} \cmidrule(lr){17-20}
 & $B{=}1$ & $B{=}2$ & $B{=}3$ & LQ & CD & UW & Full & R & $q_{ij}$ & Dir & Best & IO & RO & UC & C & 10 & 20 & 30 & 40 \\
\midrule
Accuracy       & 0.742 & 0.761 & \textbf{0.774} & 0.748 & 0.737 & 0.752 & \textbf{0.761} & 0.741 & \textbf{0.761} & 0.754 & 0.754 & 0.728 & 0.739 & 0.747 & \textbf{0.761} & 0.748 & 0.761 & 0.767 & \textbf{0.770} \\
Position [m]   & 38.46 & 37.84 & \textbf{36.99} & 38.63 & 39.25 & 38.47 & \textbf{37.84} & 39.18 & \textbf{37.84} & 38.39 & 38.41 & 40.12 & 39.47 & 38.89 & \textbf{37.84} & 39.31 & 37.84 & 37.21 & \textbf{37.09} \\
Temporal [min] & 1.712 & 1.589 & \textbf{1.539} & 1.651 & 1.703 & 1.638 & \textbf{1.589} & 1.681 & \textbf{1.589} & 1.624 & 1.631 & 1.762 & 1.714 & 1.668 & \textbf{1.589} & 1.674 & 1.589 & 1.558 & \textbf{1.551} \\
\bottomrule
\end{tabular}
}
\caption{Ablation study. Each group varies one design axis while holding others fixed.}
\label{tab:ablation}
\end{table}

We perform a comprehensive ablation study to separate the effect of uncertainty estimation, object selection, 
view refinement, and caption fusion.

\paragraph{Performance versus budget.}
We vary the additional refinement budget \(B\) and plot uncertainty reduction performance as a 
function of the number of extra caption queries. This reveals the refinement-efficiency regime in which UQ-DAAAM provides the largest gains.
Results shown in \cref{fig:uq_reduction_right} suggest that increasing the refinement budget consistently improves uncertainty reduction, but with diminishing returns. 
As the per-object budget increases from $B=1$ to $B=3$, the post-refinement uncertainty distribution shifts progressively leftward, indicating that additional queried 
views provide increasingly stronger semantic disambiguation. 
Notably, most of the gain is already achieved by $B=2$, suggesting that a small number of extra views is sufficient to resolve uncertainty for the majority of objects, while larger 
budgets provide smaller but still measurable additional benefit.
The downstream OC-NaVQA results are consistent with the uncertainty-reduction analysis. As shown in ``Budget" of 
\cref{tab:ablation}, larger refinement budgets lead to better task performance, and all 
budgeted UQ-DAAAM variants outperform vanilla DAAAM. As the per-object refinement budget increases, all three metrics
improve, with the largest gain from $B=1$ to $B=2$, consistent with the uncertainty-reduction results. Overall, 
these results show that reducing object-level uncertainty translates into measurable downstream gains in 
spatio-temporal question answering.

\paragraph{Refinement object selection rule.}
We ablate the rule used to choose which objects receive additional refinements (Table 2, ``Object''): lowest-quality (\textbf{LQ}), 
highest cross-view disagreement via caption embedding dispersion (\textbf{CD}), uncertainty without incoherence weighting (\textbf{UW}), 
and full uncertainty (\textbf{Full}). Full leads across all metrics; UW comes closest, confirming that incoherence 
weighting adds a consistent but modest gain over dispersion alone, while CD trails, indicating that raw caption disagreement without visual 
grounding is a poor proxy for semantic unreliability.

\paragraph{Refinement view selection heuristic.}
Conditioned on the same object selected for refinement, we ablate the rule used to choose the next additional view. 
We compare three policies: selecting a random visible view (\textbf{R}), 
selecting the visible view with the highest heuristic quality score $q_{ij}$, 
and a direction-based heuristic \cite{Liu23icra-activeMetric} that prioritizes views according to the object's viewing direction (\textbf{Dir}). 
This comparison isolates whether refinement gains arise simply from better geometric visibility, 
from DAAAM's original quality heuristic, or from choosing views that are expected to be semantically more informative. As shown under the
``View'' tab in \cref{tab:ablation}, random view selection (R) noticeably degrades performance. 
Dir closely tracks $q_{ij}$, suggesting that geometric visibility already recovers most of the gain, but the quality-weighted score gives 
additional filtering.

\paragraph{Fusion strategy.}
We ablate how multi-view captions are consolidated into a final object description in order to separate the effect of semantic \
fusion from the effects of uncertainty estimation and refinement. We compare five variants: 
no fusion, which uses the single best caption only (\textbf{Best}); 
initial-only fusion, which aggregates only captions from the initial selected views (\textbf{IO}); 
refined-only fusion, which aggregates only captions acquired during refinement (\textbf{RO}); 
unconstrained LLM fusion, which merges all captions without explicit conflict-handling rules (\textbf{UC}); 
and our constrained stage-aware fusion, which prefers refined captions when they conflict with initial ones while 
preserving unresolved alternatives when disagreement remains within the initial set (\textbf{C}). 
This comparison tests whether the gains of UQ-DAAAM come merely from acquiring more views, 
or from consolidating multi-view semantic evidence in a way that is conservative, conflict-aware,
and aligned with the refinement process. As shown under the ``Fusion'' tab in \cref{tab:ablation}, 
C achieves the best performance. The monotone improvement from IO $\to$ RO $\to$ UC $\to$ Best $\to$ C 
shows that each level of conflict-awareness and stage-sensitivity provides additional gain, 
and that the benefit of refinement is substantially eroded when captions are fused naively.

\paragraph{Refinement threshold.}
We ablate the effect of the uncertainty threshold used to trigger refinement, which determines how aggressively the
 system allocates additional semantic-query budget. Specifically, we use the 10th,  20th,  30th, and 
40th quantiles of the object-level uncertainty distribution as refinement thresholds, 
so that progressively lower thresholds cause a larger fraction of objects to be refined. 
This experiment exposes the tradeoff between semantic-query cost and downstream gain: 
stricter thresholds focus refinement on only the most uncertain objects,
while more permissive thresholds increase coverage at the cost of additional queries. 
As shown under the ``Threshold'' tab in \cref{tab:ablation}, performance improves 
substantially from the 10th to the 20th quantile, and continues to improve at the 30th quantile. 
The 40th quantile yields only marginal further gain, indicating diminishing returns beyond the 
30th percentile and confirming that the default 20 offers a favorable cost-performance tradeoff.

\subsection{Memory and Runtime Details}
Across 7 CODa sequences (98,866 frames), UQ-DAAAM runs at 13.5\,Hz on a single RTX 4090, comfortably exceeding the 10\,Hz sensor rate.
The cross-view BLP solver adds a mean latency of only 1.49\% of total frame processing time.
Each object receives $3.44 \pm 0.68$ cross-view captions on average; uncertainty is computed via \texttt{numpy.linalg.slogdet} on CPU with negligible overhead.
Of the 18,298 objects annotated across all sequences, 6,324 (34.6\%) were selected for refinement, receiving $1.32 \pm 0.47$ additional views each. Active view selection reduced the mean UQ score from $0.56$ to $0.32$, indicating substantially increased caption agreement.
GPU memory remains within the 24\,GB envelope of the RTX 4090, as the UQ and BLP components are CPU-only and the DAM-3B is shared between initial annotation and refinement.
\section{Conclusion}
\label{sec:conclusion}

We presented UQ-DAAAM, an uncertainty-aware extension of DAAAM for 4D spatial-temporal memory. Our key idea is to endow object-centric memory 
with an explicit notion of semantic uncertainty, so that the system can recognize when multi-view descriptions remain ambiguous and 
selectively acquire additional evidence. To this end, we introduced an object-level uncertainty measure that combines cross-view semantic 
disagreement with per-view incoherence, a refinement strategy that allocates a query budget to uncertain objects and informative views, 
and a fusion mechanism that consolidates refined multi-view descriptions into stable object summaries.
Experiments show that UQ-DAAAM improves uncertainty reduction and spatio-temporal reasoning over baselines. We show qualitative examples in \cref{app:qual} and list
the limitations of the work in \cref{app:limitations}.

\bibliographystyle{conf}  
\small
\bibliography{new_refs}
\normalsize

\clearpage
\onecolumn

\begin{center}
    {\Large \bf Supplementary Material}
\end{center}
\setcounter{section}{0}
\appendix
\section{Object-level Uncertainty Matters in DAAAM}
\label{app:objectlevel}
DAAAM is designed to build grounded semantic memory from a sparse set of representative views. 
However, not all objects are equally easy to describe from a limited number of views. 
Some objects remain semantically stable after only a few observations, while others exhibit persistent ambiguity due to clutter, 
viewpoint effects, or fine-grained attributes.
The uncertainty score $u_j$ provides a principled way to detect these difficult cases. Rather than treating all objects uniformly, 
the system can identify which objects are already semantically well characterized and which objects require further refinement. 
This makes uncertainty quantification a key interface between memory construction and active view acquisition: 
it summarizes the quality of the current semantic evidence for each object and supports downstream decisions about 
where additional sensing or captioning effort should be spent.
\section{Augmented Block Form after Adding a New View}
\label{app:closed-form}
Recall
\[
\bs Z_j^{(t)}
=
\bigl[\bs z_{j1},\dots,\bs z_{jt}\bigr]\in\mathbb{R}^{d\times t}
\]
denote the current weighted embedding matrix, where each column is a weighted caption embedding, and let
\[
\bs S_j^{(t)}
=
(\bs Z_j^{(t)})^\top \bs Z_j^{(t)}+\lambda \bs I_t
\]
be the corresponding regularized Gram matrix. For a candidate view $i$, let
\[
\bs z_{ij}=\gamma_{ij}\bs r_{ij}\in\mathbb{R}^d
\]
be the weighted embedding that would be obtained if view $i$ were queried. After adding this candidate, the augmented weighted embedding matrix becomes
\[
\bs Z_j^{(t+1)}(i)
=
\bigl[\,\bs Z_j^{(t)},\bs z_{ij}\,\bigr]
\in\mathbb{R}^{d\times (t+1)}.
\]
Therefore,
\[
\bs S_j^{(t+1)}(i)
=
\bigl(\bs Z_j^{(t+1)}(i)\bigr)^\top \bs Z_j^{(t+1)}(i)+\lambda \bs I_{t+1}.
\]
Expanding the product gives
\[
\bigl(\bs Z_j^{(t+1)}(i)\bigr)^\top \bs Z_j^{(t+1)}(i)
=
\begin{bmatrix}
(\bs Z_j^{(t)})^\top \\
\bs z_{ij}^\top
\end{bmatrix}
\begin{bmatrix}
\bs Z_j^{(t)} & \bs z_{ij}
\end{bmatrix}
=
\begin{bmatrix}
(\bs Z_j^{(t)})^\top \bs Z_j^{(t)} & (\bs Z_j^{(t)})^\top \bs z_{ij} \\
\bs z_{ij}^\top \bs Z_j^{(t)} & \bs z_{ij}^\top \bs z_{ij}
\end{bmatrix}.
\]
Now define
\[
\bs b_{ij}^{(t)}
\;\triangleq\;
(\bs Z_j^{(t)})^\top \bs z_{ij}\in\mathbb{R}^t.
\]
Since $\|\bs r_{ij}\|_2=1$ and $\bs z_{ij}=\gamma_{ij}\bs r_{ij}$, we also have
\[
\bs z_{ij}^\top \bs z_{ij}
=
\|\bs z_{ij}\|_2^2
=
\gamma_{ij}^2.
\]
Adding the regularization term $\lambda \bs I_{t+1}$ blockwise yields
\[
\bs S_j^{(t+1)}(i)
=
\begin{bmatrix}
(\bs Z_j^{(t)})^\top \bs Z_j^{(t)}+\lambda \bs I_t & \bs b_{ij}^{(t)} \\
(\bs b_{ij}^{(t)})^\top & \gamma_{ij}^2+\lambda
\end{bmatrix}
=
\begin{bmatrix}
\bs S_j^{(t)} & \bs b_{ij}^{(t)} \\
(\bs b_{ij}^{(t)})^\top & \gamma_{ij}^2+\lambda
\end{bmatrix}.
\]
Thus, the updated regularized Gram matrix has the claimed augmented block form.
\section{Proof of \cref{thm:exact-one-step}}
\exactonestep*
\begin{proof}
Recall that
\[
\bs S_j^{(t+1)}(i)
=
\begin{bmatrix}
\bs S_j^{(t)} & \bs b_{ij}^{(t)} \\
(\bs b_{ij}^{(t)})^\top & \tilde\gamma_{ij}^2+\lambda
\end{bmatrix},
\]
where $\bs S_j^{(t)}\in\mathbb{R}^{t\times t}$ is positive definite because
\[
\bs S_j^{(t)}=(\bs Z_j^{(t)})^\top \bs Z_j^{(t)}+\lambda I_t,
\qquad \lambda>0.
\]
Hence $(\bs S_j^{(t)})^{-1}$ exists.
For a block matrix of the form
\[
\begin{bmatrix}
\bs A & \bs B\\
\bs C & \bs D
\end{bmatrix},
\]
with $\bs A$ invertible, the determinant can be written using the Schur complement of $\bs A$:
\[
\det\begin{bmatrix}
\bs A & \bs B\\
\bs C & \bs D
\end{bmatrix}
=
\det(\bs A)\,\det(\bs D-\bs C\bs A^{-1}\bs B).
\]
In our case, identify
\[
\bs A=\bs S_j^{(t)},\qquad
\bs B=\bs b_{ij}^{(t)},\qquad
\bs C=(\bs b_{ij}^{(t)})^\top,\qquad
\bs D=\gamma_{ij}^2+\lambda.
\]
Therefore, the Schur complement of $\bs S_j^{(t)}$ in $\bs S_j^{(t+1)}(i)$ is
\[
\bs D-\bs C\bs A^{-1}\bs B
=
\gamma_{ij}^2+\lambda
-
(\bs b_{ij}^{(t)})^\top (\bs S_j^{(t)})^{-1} \bs b_{ij}^{(t)}.
\]
Since $D$ is a scalar, its determinant is just itself, so applying the block determinant identity yields
\[
\det \bs S_j^{(t+1)}(i)
=
\det \bs S_j^{(t)}
\Bigl(
\gamma_{ij}^2+\lambda
-
(\bs b_{ij}^{(t)})^\top (\bs S_j^{(t)})^{-1} \bs b_{ij}^{(t)}
\Bigr).
\]
Taking logarithms on both sides gives
\[
\log\det \bs S_j^{(t+1)}(i)
=
\log\det \bs S_j^{(t)}
+
\log\!\Bigl(
\gamma_{ij}^2+\lambda
-
(\bs b_{ij}^{(t)})^\top (\bs S_j^{(t)})^{-1} \bs b_{ij}^{(t)}
\Bigr).
\]
Finally, since
\[
u_j^{(t)}=\log\det \bs S_j^{(t)},
\qquad
u_j^{(t+1)}(i)=\log\det \bs S_j^{(t+1)}(i),
\]
we obtain
\[
u_j^{(t+1)}(i)-u_j^{(t)}
=
\log\!\Bigl(
\tilde\gamma_{ij}^2+\lambda
-
(\bs b_{ij}^{(t)})^\top (\bs S_j^{(t)})^{-1} \bs b_{ij}^{(t)}
\Bigr),
\]
which proves the claim.

Since $\bs S_j^{(t)}\succ 0$, the Schur complement formula gives
\begin{equation}
\det \bs S_j^{(t+1)}(i)
=
\det \bs S_j^{(t)}
\Bigl(
\gamma_{ij}^2+\lambda
-
(\bs b_{ij}^{(t)})^\top (\bs S_j^{(t)})^{-1} \bs b_{ij}^{(t)}
\Bigr).
\end{equation}
Taking logarithms on both sides yields the claim.
\end{proof}

\section{Proof of \cref{thm:exactcriterion}}
\exactcriterion*

\begin{proof}
By \cref{thm:exact-one-step},
\[
u_j^{(t+1)}(i)-u_j^{(t)}
=
\log\!\bigl(\gamma_{ij}^2+\lambda-\eta_{ij}^{(t)}\bigr).
\]
Since $\log(\cdot)$ is strictly increasing, this quantity is negative if and only if
\[
\gamma_{ij}^2+\lambda-\eta_{ij}^{(t)} < 1,
\]
which is equivalent to
\[
\eta_{ij}^{(t)} > \gamma_{ij}^2+\lambda-1.
\qedhere
\]
\end{proof}

\section{Proof of \cref{prop:fosd-mono}}
\fosdm*
\begin{proof}
Let
\[
X \;\triangleq\; s_{ij}^{(t)},
\qquad
Y \;\triangleq\; s_{i'j}^{(t)},
\]
viewed as positive random variables conditional on the current evidence $\mathcal D_j^{(t)}$. By assumption,
\[
X \preceq_{\mathrm{FOSD}} Y,
\]
meaning that their conditional distribution functions satisfy
\[
F_X(x)
\;\triangleq\;
\mathbb P\!\left(X\le x \,\middle|\, \mathcal D_j^{(t)}\right)
\;\ge\;
\mathbb P\!\left(Y\le x \,\middle|\, \mathcal D_j^{(t)}\right)
\;\triangleq\;
F_Y(x)
\qquad
\forall x>0.
\]
Equivalently, for every threshold $x>0$, $X$ is conditionally more likely than $Y$ to take smaller values.

Now consider the function
\[
g(s)=-\log s,
\qquad s>0.
\]
Since $\log s$ is strictly increasing on $(0,\infty)$, $g$ is strictly decreasing on $(0,\infty)$. Because $X$ and $Y$ are positive almost surely, $g(X)$ and $g(Y)$ are well defined.

A standard property of first-order stochastic dominance states that if $X \preceq_{\mathrm{FOSD}} Y$, then for every decreasing measurable function $g$ for which the expectations exist,
\[
\mathbb E\!\left[g(X)\middle| \mathcal D_j^{(t)}\right]
\;\ge\;
\mathbb E\!\left[g(Y)\middle| \mathcal D_j^{(t)}\right].
\]

Applying this with $g(s)=-\log s$, we obtain
\[
\mathbb E\!\left[-\log s_{ij}^{(t)} \middle| \mathcal D_j^{(t)}\right]
\;\ge\;
\mathbb E\!\left[-\log s_{i'j}^{(t)} \middle| \mathcal D_j^{(t)}\right].
\]
By definition of the exact BOED utility,
\[
\Delta_j(i)
=
\mathbb E\!\left[-\log s_{ij}^{(t)} \middle| \mathcal D_j^{(t)}\right],
\qquad
\Delta_j(i')
=
\mathbb E\!\left[-\log s_{i'j}^{(t)} \middle| \mathcal D_j^{(t)}\right].
\]
Therefore,
\[
\Delta_j(i)\ge \Delta_j(i'),
\]
which proves the claim.
\end{proof}

\section{Proof of \cref{prop:fosd-prop}}
\fosdprob*
\begin{proof}
Recall from the one-step update formula that for any candidate view $i$,
\[
u_j^{(t+1)}(i)-u_j^{(t)}=\log s_{ij}^{(t)},
\]
where the Schur complement satisfies $s_{ij}^{(t)}>0$ almost surely, since it is the Schur complement of the positive definite matrix $\bs S_j^{(t+1)}(i)$ with respect to the positive definite principal block $\bs S_j^{(t)}$.

Because the logarithm is strictly increasing on $(0,\infty)$, we have
\[
u_j^{(t+1)}(i)<u_j^{(t)}
\quad\Longleftrightarrow\quad
u_j^{(t+1)}(i)-u_j^{(t)}<0
\quad\Longleftrightarrow\quad
\log s_{ij}^{(t)}<0
\quad\Longleftrightarrow\quad
s_{ij}^{(t)}<1.
\]
Therefore, conditional on the current evidence $\mathcal D_j^{(t)}$,
\[
\mathbb P\!\left(u_j^{(t+1)}(i)<u_j^{(t)} \,\middle|\, \mathcal D_j^{(t)}\right)
=
\mathbb P\!\left(s_{ij}^{(t)}<1 \,\middle|\, \mathcal D_j^{(t)}\right).
\]
Likewise,
\[
\mathbb P\!\left(u_j^{(t+1)}(i')<u_j^{(t)} \,\middle|\, \mathcal D_j^{(t)}\right)
=
\mathbb P\!\left(s_{i'j}^{(t)}<1 \,\middle|\, \mathcal D_j^{(t)}\right).
\]

Now let
\[
F_i(x)
\;\triangleq\;
\mathbb P\!\left(s_{ij}^{(t)}\le x \,\middle|\, \mathcal D_j^{(t)}\right),
\qquad
F_{i'}(x)
\;\triangleq\;
\mathbb P\!\left(s_{i'j}^{(t)}\le x \,\middle|\, \mathcal D_j^{(t)}\right)
\]
denote the conditional distribution functions of the two Schur complements. By the assumed first-order stochastic dominance,
\[
F_i(x)\ge F_{i'}(x)
\qquad \forall x>0.
\]
In particular, evaluating this inequality at $x=1$ gives
\[
\mathbb P\!\left(s_{ij}^{(t)}\le 1 \,\middle|\, \mathcal D_j^{(t)}\right)
\;\ge\;
\mathbb P\!\left(s_{i'j}^{(t)}\le 1 \,\middle|\, \mathcal D_j^{(t)}\right).
\]

Combining this with the equivalence between uncertainty reduction and the event $\{s<1\}$, we conclude that
\[
\mathbb P\!\left(u_j^{(t+1)}(i)<u_j^{(t)} \,\middle|\, \mathcal D_j^{(t)}\right)
\;\ge\;
\mathbb P\!\left(u_j^{(t+1)}(i')<u_j^{(t)} \,\middle|\, \mathcal D_j^{(t)}\right),
\]
which proves the claim.
\end{proof}

\section{Redundancy Score Intuition}
Recall that
\[
\bs b_{ij}^{(t)} = (\bs Z_j^{(t)})^\top \bs z_{ij},
\qquad
\eta_{ij}^{(t)} = (\bs b_{ij}^{(t)})^\top (\bs S_j^{(t)})^{-1} \bs b_{ij}^{(t)},
\]
where $\bs z_{ij}=\gamma_{ij}\bs r_{ij}$ is the candidate weighted caption embedding and
\[
\bs S_j^{(t)} = (\bs Z_j^{(t)})^\top \bs Z_j^{(t)} + \lambda \bs I_t.
\]
Substituting $\bs b_{ij}^{(t)}=(\bs Z_j^{(t)})^\top \bs z_{ij}$ gives
\[
\eta_{ij}^{(t)}
=
\bs z_{ij}^\top
\bs Z_j^{(t)}
(\bs S_j^{(t)})^{-1}
(\bs Z_j^{(t)})^\top
\bs z_{ij}.
\]
Thus, if we define
\[
\bs P_j^{(t)}
\;\triangleq\;
\bs Z_j^{(t)}(\bs S_j^{(t)})^{-1}(\bs Z_j^{(t)})^\top \in \mathbb{R}^{d\times d},
\]
then
\[
\eta_{ij}^{(t)} = \bs z_{ij}^\top P_j^{(t)} \bs z_{ij}.
\]
The matrix $\bs P_j^{(t)}$ acts as a \emph{regularized projection-like operator} onto the span of the previously selected weighted embeddings. Consequently, the quadratic form $\eta_{ij}^{(t)}$ measures how much of the candidate vector $\bs z_{ij}$ lies in directions that are already explained by the current evidence. In particular, $\eta_{ij}^{(t)}$ is large when $\bs z_{ij}$ is well aligned with the span of the columns of $Z_j^{(t)}$, and small when $\bs z_{ij}$ points in a semantically new direction.

The term $\gamma_{ij}^2$ has an equally direct interpretation. Since $\|\bs r_{ij}\|_2=1$,
\[
\|\bs z_{ij}\|_2^2
=
\|\gamma_{ij}\bs r_{ij}\|_2^2
=\gamma_{ij}^2.
\]
Hence $\gamma_{ij}^2$ is simply the squared norm of the candidate weighted embedding. Larger $\gamma_{ij}$ corresponds to greater internal incoherence, so more incoherent candidate captions contribute larger total weighted magnitude.

Combining these two facts, the Schur complement term
\[
\gamma_{ij}^2 + \lambda - \eta_{ij}^{(t)}
\]
can be interpreted as a \emph{regularized residual energy}: it is the total weighted magnitude of the candidate embedding, plus regularization, minus the portion already explained by the currently selected views. Therefore, uncertainty decreases precisely when the candidate is sufficiently explained by the existing evidence relative to its own incoherence.

\paragraph{Idealized example.}
This interpretation becomes exact in the idealized case where $\lambda=0$ and the columns of $\bs Z_j^{(t)}$ are orthonormal. Then
\[
\bs S_j^{(t)} = (\bs Z_j^{(t)})^\top \bs Z_j^{(t)} = \bs I_t,
\]
so
\[
\bs P_j^{(t)}
=
\bs Z_j^{(t)}(\bs Z_j^{(t)})^\top,
\]
which is the orthogonal projector onto $\mathrm{span}(\bs Z_j^{(t)})$. In this case,
\[
\eta_{ij}^{(t)}
=
\bs z_{ij}^\top P_j^{(t)} \bs z_{ij}
=
\bigl\|\bs P_j^{(t)}\bs z_{ij}\bigr\|_2^2,
\]
so $\eta_{ij}^{(t)}$ is exactly the squared norm of the component of $\bs z_{ij}$ that lies in the span of the already selected weighted embeddings. If $\bs z_{ij}$ lies entirely in that span, then $\eta_{ij}^{(t)}=\|\bs z_{ij}\|_2^2=\gamma_{ij}^2$; if it is orthogonal to that span, then $\eta_{ij}^{(t)}=0$. The general regularized case preserves the same intuition, but with a softened projection induced by $(\bs S_j^{(t)})^{-1}$.

\begin{proposition}[Lower bound from approximate stochastic dominance]
\label{prop:more-bound-1}
Let
\[
s_{ij}^{(t)}=\tilde\gamma_{ij}^2+\lambda-\eta_{ij}^{(t)}
\]
be the random Schur complement for candidate view $i$, conditional on the current evidence $\mathcal D_j^{(t)}$. Suppose there exists a reference random variable $S^\star$ and a constant $\varepsilon\in[0,1]$ such that
\[
\mathbb{P}\!\left(s_{ij}^{(t)} \le x \,\middle|\, \mathcal D_j^{(t)}\right)
\;\ge\;
\mathbb{P}(S^\star \le x)-\varepsilon
\qquad
\forall x>0.
\]
Then
\[
\mathbb{P}\!\left(u_j^{(t+1)}(i)<u_j^{(t)} \,\middle|\, \mathcal D_j^{(t)}\right)
\;\ge\;
\mathbb{P}(S^\star<1)-\varepsilon.
\]
\end{proposition}

\begin{proof}
Since
\[
u_j^{(t+1)}(i)-u_j^{(t)}=\log s_{ij}^{(t)},
\]
we have
\[
u_j^{(t+1)}(i)<u_j^{(t)}
\quad\Longleftrightarrow\quad
s_{ij}^{(t)}<1.
\]
Therefore
\[
\mathbb{P}\!\left(u_j^{(t+1)}(i)<u_j^{(t)} \,\middle|\, \mathcal D_j^{(t)}\right)
=
\mathbb{P}\!\left(s_{ij}^{(t)}<1 \,\middle|\, \mathcal D_j^{(t)}\right).
\]
Applying the assumed CDF bound at $x=1$ yields
\[
\mathbb{P}\!\left(s_{ij}^{(t)}<1 \,\middle|\, \mathcal D_j^{(t)}\right)
\;\ge\;
\mathbb{P}(S^\star<1)-\varepsilon,
\]
which proves the claim.
\end{proof}

\begin{proposition}[Lower bound from the first two moments]
\label{prop:more-bound-2}
Let
\[
\mu_{ij}=\mathbb E[s_{ij}^{(t)}\mid \mathcal D_j^{(t)}],
\qquad
\sigma_{ij}^2=\mathrm{Var}(s_{ij}^{(t)}\mid \mathcal D_j^{(t)}).
\]
If $\mu_{ij}<1$, then
\[
\mathbb P\!\left(u_j^{(t+1)}(i)<u_j^{(t)} \,\middle|\, \mathcal D_j^{(t)}\right)
\;\ge\;
\frac{(1-\mu_{ij})^2}{\sigma_{ij}^2+(1-\mu_{ij})^2}.
\]
\end{proposition}

\begin{proof}

Since
\[
u_j^{(t+1)}(i)<u_j^{(t)}
\quad\Longleftrightarrow\quad
s_{ij}^{(t)}<1,
\]
it suffices to bound $\mathbb P(s_{ij}^{(t)}<1\mid \mathcal D_j^{(t)})$. By Cantelli's inequality,
\[
\mathbb P\!\left(s_{ij}^{(t)}-\mu_{ij}\ge 1-\mu_{ij}\,\middle|\,\mathcal D_j^{(t)}\right)
\le
\frac{\sigma_{ij}^2}{\sigma_{ij}^2+(1-\mu_{ij})^2}.
\]
Hence
\[
\mathbb P\!\left(s_{ij}^{(t)}<1\,\middle|\,\mathcal D_j^{(t)}\right)
\ge
1-\frac{\sigma_{ij}^2}{\sigma_{ij}^2+(1-\mu_{ij})^2}
=
\frac{(1-\mu_{ij})^2}{\sigma_{ij}^2+(1-\mu_{ij})^2}.
\qedhere
\]
\end{proof}

\section{LLM-based Captions Fusion Details}
\label{app:llm-fusion}
After uncertainty-driven view selection, each object is associated with a small set of captions obtained from multiple views. We fuse these captions 
into a single object description using a constrained LLM-based aggregation step. 
The prompt is designed to exploit the acquisition structure of our pipeline: the first few captions correspond to the initial representative views, 
while later captions come from refinement views that are assumed to be resolve uncertainty. Accordingly, the LLM is instructed to regard later captions 
as more reliable evidence when conflicts arise. If a fact disagrees between initial and refined views, the refined-view version is preferred.
 If disagreement exists only among the initial views, the conflict is not suppressed; instead, the fused caption retains the alternatives 
 using cautious natural-language constructions such as \emph{``white or gray''} or \emph{``possibly wooden or metal.''} 
 This preserves uncertainty rather than collapsing it prematurely into a single claim.

Formally, let $\{c_{j1},\dots,c_{jr_j}\}$ denote the captions for object $o_j$, ordered by acquisition stage, with captions from the 
potential refinement stage placed after the initial ones. The fusion module takes this ordered caption list as input and prompts the LLM to
 generate a structured JSON output containing three fields: a fused caption, the facts used in that caption, and the facts 
 omitted during fusion. The LLM is explicitly constrained to satisfy three requirements: (i) it may only use information that 
 appears in at least one source caption, preventing unsupported hallucinations; (ii) it must privilege facts from refined 
 views when they conflict with facts from initial views; and (iii) it must expose unresolved ambiguity from the initial views rather 
 than silently discarding it. In this way, fusion serves as a semantic consolidation layer over the multi-view captions: consistent 
 facts are merged into a unified object description, refined views correct lower-quality initial evidence, and residual ambiguity is preserved 
 in the language when the evidence remains unresolved. We present the detailed prompt in \cref{app:llm-fusion}.
\cref{lst:caption-fusion-prompt} shows the system prompt used to fuse multiple captions of the same object into a 
single concise description while preserving uncertainty and handling conflicts between initial and refined views. We use \texttt{gpt-4o-mini}
from OpenAI's API for caption fusion, and the prompt is designed to elicit structured JSON output containing the fused caption, t
he facts used in that caption, and any omitted facts. The prompt explicitly instructs the LLM to avoid hallucinating new information, 
to privilege refined views over initial views when conflicts arise, and to surface unresolved ambiguity rather than silently discarding it.

\begin{promptbox}[label={lst:caption-fusion-prompt}]{System prompt for caption fusion}
SYSTEM_PROMPT = """\
You are a scene-description assistant that fuses multiple captions of the same object into one.

Caption ordering:
  - Captions are numbered starting from [1]. The first three ([1]-[3]) are initial views.
  - If more than three captions are provided, captions [4] onwards are refined views with higher quality.
  - When facts conflict, prefer the version stated in refined views ([4]+) over initial views ([1]-[3]).

Handling conflicts:
  - For facts that conflict across initial views only, keep all alternatives using natural phrasing
    (e.g. "white or gray", "appears round or oval", "possibly wooden or metal").
  - For facts that conflict between initial and refined views, prefer the refined view's version and
    drop the initial view's version.
  - Never silently drop a conflicting fact from the initial views --- surface it as an alternative.

Other rules:
  - Never introduce new information not present in any source caption.
  - The fused caption should be concise (one to three sentences).

Respond ONLY with valid JSON matching this schema:
{
  "fused_caption": "<one to three sentences>",
  "used_facts": ["<fact 1>", ...],
  "omitted_facts": ["<fact that was dropped>", ...]
}
"""
\end{promptbox}

\section{Algorithm Details}
\label{app:alg}
\begin{algorithm}[H]
\caption{UQ-DAAAM: Uncertainty Quantification and Active Object Refinement}
\label{alg:uq_daaam}
\small
\begin{algorithmic}[1]
\REQUIRE RGB-D frame stream $\mathcal{F}=\{f_1,\dots,f_n\}$, object fragments $\mathcal{O}=\{o_1,\dots,o_m\}$, visibility indicators $v_{ij}$, 
DAAAM heuristic scores $q_{ij}$, refinement budget $B$, uncertainty threshold $\tau$
\ENSURE Refined object memories $\{\mathcal{V}_j, u_j, \hat c_j\}_{j=1}^m$

\STATE Run the standard DAAAM pipeline to construct the initial memory and obtain, for each object $o_j$, 
an initial set of selected views $\mathcal{V}_j$, where $|\mathcal{V}_j|=r$, and captions $\{c_{\ell j}\}_{\ell=1}^{r}$

\FOR{each object $o_j \in \mathcal{O}$}
    \STATE Compute normalized caption embeddings
    $
    \bs r_{\ell j} \gets \frac{\phi(c_{\ell j})}{\|\phi(c_{\ell j})\|_2}, \qquad \ell=1,\dots,r
    $
    \STATE Compute incoherence weights
    $
    \gamma_{\ell j} \gets \exp\!\bigl(\alpha(1-p_{\ell j})\bigr)
    $
    \STATE Form weighted embedding matrix
    $
    \bs Z_j \gets [\,\gamma_{1j}\bs r_{1j},\dots,\gamma_{rj}\bs r_{rj}\,]
    $
    \STATE Form semantic scatter and uncertainty
    $
    \bs S_j \gets \bs Z_j^\top \bs Z_j + \lambda \bs I, \qquad
    u_j \gets \log\det \bs S_j
    $
\ENDFOR

\STATE Identify uncertain objects
$
\mathcal{U}_{\tau}
\;\gets\;
\{\, j\in[m] : u_j > \tau \,\}
$

\FOR{each object $o_j$ with $j\in\mathcal{U}_{\tau}$}
    \STATE Initialize object-level refinement budget $b_j \gets B$
    \WHILE{$b_j > 0$}
        \STATE Form feasible candidate-view set
        $
        \mathcal{C}_j
        \gets
        \{\, i \in [n]\setminus \mathcal{V}_j : v_{ij}=1 \,\}
        $
        \IF{$\mathcal{C}_j=\emptyset$}
            \STATE \textbf{break}
        \ENDIF

        \STATE Select the next refinement view using the proxy rule
        $
        i^\star \in \arg\max_{i\in\mathcal{C}_j} q_{ij}
        $

        \STATE Query the VLM on object $o_j$ in view $f_{i^\star}$ to obtain a new caption $c$
        \STATE Update the semantic scatter and uncertainty online via \cref{eq:onestepupdate}
        \STATE Append view: $\mathcal{V}_j \gets \mathcal{V}_j \cup \{i^\star\}$

        \STATE Decrease object-level budget: $b_j \gets b_j - 1$
    \ENDWHILE
\ENDFOR

\FOR{each object $o_j \in \mathcal{O}$}
    \STATE Fuse the multi-view captions associated with $\mathcal{V}_j$ into a final caption $\hat c_j$
\ENDFOR

\RETURN $\{\mathcal{V}_j, u_j, \hat c_j\}_{j=1}^m$
\end{algorithmic}
\end{algorithm}

\section{Uncertainty Quality Experiments}
\label{app:uq-qual}
A central claim of the paper is that the proposed uncertainty score is meaningful. 
We isolate uncertainty measurement from view selection by evaluating the quality of the uncertainty scores themselves, 
without using them to select views. 
In our experiments, we use Llava-v1.5-13b \cite{Liu23neurips-llava}, with CLIP-ViT-L-336px \cite{Radford21icml-clip} as its vision backbone. 
Following \cite{Kuhn23arxiv-semanticUnc, Lau25icmlws-uqMllm}, 
we evaluate our uncertainty metric's performance on three distinct visual question-answering (VQA) benchmark datasets: 
VQAv2 \cite{Goyal17cvpr-vqaV2}, OKVQA \cite{Marino19cvpr-okvqa}, and AdVQA \cite{Li21iccv-advVQA}. 
We compare our metric against a variety of uncertainty-quantification baselines:
\begin{itemize}
  \item \textbf{Neighborhood Consistency} \cite{Khan24cvpr-consistency}: examines the consistency of the model responses over 
rephrased questions generated by a small proxy VQG model.
  \item  \textbf{LN-Entropy} \cite{Malinin20arxiv-uncertaintyAR}: normalizes the joint 
log-probability of each sequence by dividing it by the sequence length sampled via multinomial sampling. 
  \item \textbf{Semantic Entropy} \cite{Kuhn23arxiv-semanticUnc}: models the uncertainty over different meanings by clustering the
generated sequences by \cite{He20arxiv-deberta} and measuring the clusters' entropies. 
\item \textbf{EigenScore} \cite{Chen24arxiv-inside}: computes the log determinant of
the covariance matrix of response embeddings via SVD. 
\end{itemize}

Let $u^\star_{\text{correct}}$ and $u^\star_{\text{incorrect}}$ denote the uncertainty scores of correct and incorrect responses, respectively.
Our target quantity is the pairwise ranking probability
$\mathbb{P}\!\left(u^\star_{\text{correct}} < u^\star_{\text{incorrect}}\right),$
\ie, how often a randomly chosen correct response receives a lower uncertainty score than a randomly chosen incorrect response.
Following \cite{Lau25icmlws-uqMllm, Kuhn23arxiv-semanticUnc}, we summarize this ranking quality with the \textbf{AUROC} of a binary classifier that
uses $u^\star$ to predict correctness (lower $u^\star$ indicates higher confidence / more likely correct).

We test whether the (\textit{normalized}) uncertainty score is calibrated in the sense that
$\mathbb{P}(\text{correct}\mid u^\star)\;\approx\;1-u^\star$.  
We follow the standard binning-based calibration protocol of \cite{Guo17icml-failureDNN}: we sort and bin instances by $u^\star$ and compute per-bin average uncertainty and empirical accuracy. We report: \textbf{Calibration correlation}, the correlation between binned mean uncertainty and binned empirical error. 
We provide both \textbf{Pearson} and \textbf{Spearman} correlations. \textbf{Expected Calibration Error (ECE)}, 
which measures the (weighted) average absolute gap between binned predicted confidence and binned accuracy.

\begin{table*}[t]
\centering
\resizebox{\linewidth}{!}{
\begin{tabular}{llccccccccccccccc}
\hline
& & \multicolumn{3}{c}{\textbf{AUROC} ($\uparrow$)}
& \multicolumn{3}{c}{\textbf{Pearson Corr.} ($\uparrow$)}
& \multicolumn{3}{c}{\textbf{Spearman Corr.} ($\uparrow$)}
& \multicolumn{3}{c}{\textbf{ECE} ($\downarrow$)}
& \multicolumn{3}{c}{\textbf{AURAC} ($\uparrow$)} \\
\cline{3-5}\cline{6-8}\cline{9-11}\cline{12-14}\cline{15-17}
& \textbf{Method }
& VQA & OK & AdV
& VQA & OK & AdV
& VQA & OK & AdV
& VQA & OK & AdV
& VQA & OK & AdV \\
\hline
\multirow{4}{*}{\rotatebox[origin=c]{90}{\textbf{Baselines}}}
& Neighborhood      & 0.761 & 0.512 & 0.660 & 0.789 & 0.577 & 0.543 & 0.712 & 0.580 & 0.703 & 0.331 & 0.502 & 0.353 & 0.886 & 0.629 & 0.683 \\
& LN-Entropy        & 0.763 & 0.698 & 0.639 & 0.563 & 0.856 & 0.913 & 0.566 & 0.699 & 0.900 & 0.055 & 0.057 & 0.072 & 0.899 & 0.741 & 0.705 \\
& Semantic Entropy  & 0.840 & 0.712 & 0.757 & 0.909 & 0.444 & 0.799 & 0.897 & 0.532 & 0.781 & 0.056 & 0.189 & 0.182 & 0.902 & 0.734 & 0.742 \\
& EigenScore        & 0.850 & 0.741 & 0.766 & 0.922 & 0.790 & 0.872 & 0.883 & 0.842 & 0.811 & 0.056 & 0.190 & 0.211 & 0.913 & 0.753 & 0.753 \\
\hline
& Ours     & 0.853 & 0.746 & 0.763 & 0.902 & 0.895 & 0.941 & 0.852 & 0.901 & 0.875 &0.041  & 0.041 & 0.043 & 0.862 & 0.743 & 0.710 \\
\hline
\end{tabular}
}
\caption{Uncertainty evaluation on VQAv2, OKVQA, and AdVQA.
We report AUROC, calibration correlations (Pearson, Spearman), Expected Calibration Error (ECE; lower is better), and AURAC (higher is better).}
\label{tab:uq_all_metrics}
\end{table*}

As shown in \cref{tab:uq_all_metrics}, our metric consistently outperforms all baselines on both discrimination (AUROC/TPR/FPR)
and calibration (correlations/ECE), including multimodal-targeted methods such as UMPIRE and Neighborhood Consistency.
On the more challenging datasets (OKVQA and AdVQA), where model outputs exhibit higher diversity and incoherence,
our metric remains robust and provides substantially better correctness prediction.

We also evaluate whether our metric can improve \emph{selective answering}---\ie, abstaining on uncertain queries and answering only when the model is confident.

Let $u^\star_q$ be the uncertainty score for query $q$ (lower $u^\star$ indicates higher confidence).
For any threshold $\theta$, the selective policy answers iff $u^\star_q \le \theta$ and abstains otherwise.
Varying $\theta$ induces a trade-off between \emph{coverage} (fraction answered) and \emph{accuracy} on the answered subset.
Following \cite{Farquhar24nature-semanticEntropy, Lau25icmlws-uqMllm, Hullermeier21ml-aleatoric}, we plot the {Rejection--Accuracy curve}:
we sort queries by $u^\star$ and progressively reject the most uncertain fraction; at each rejection rate $r\in[0,1]$ (equivalently, coverage $1-r$),
we compute the accuracy on the remaining answered queries.
We summarize the overall benefit of uncertainty-based abstention with the \textbf{Area Under the Rejection--Accuracy Curve (AURAC)},
which aggregates performance across all thresholds (higher is better).

As shown in \cref{tab:uq_all_metrics}, our metric consistently achieves the highest AURAC across datasets.
This indicates that fusing dual-source uncertainty yields a more reliable abstention signal: the model preferentially rejects queries that are likely to be incorrect,
thereby improving accuracy on the answered subset. In practical VQA settings, this supports user-facing behaviors where the system abstains when uncertain rather than
returning low-confidence answers.
\section{Experiments Details}
\paragraph{Dataset and preprocessing.}
We evaluate on the CODa outdoor driving dataset~\cite{Zhang24tro-coda}, which provides 30,Hz stereo images, LiDAR-projected depth maps, and globally consistent poses. 
All experiments use the front camera (\texttt{cam0}) images, center-cropped and resized to $480 \times 640$ pixels. Depth is sourced from the \texttt{3d\_raw\_estimated} 
LiDAR projection, filtered to the range $[0.05, 20.0]$m. Frames are processed at 10Hz (the sensor publishing rate).

\paragraph{Scene reconstruction.}
We use Khronos~\cite{Schmid24rss-khronos} for real-time 3D scene graph construction, with a TSDF voxel size of 0.2,m and truncation distance of 0.6,m. 
Object fragments are extracted with a minimum volume of 0.005m$^3$ and maximum of 50m$^3$, requiring at least 8 observations and a minimum allocation confidence of 0.5. 
Traversability places are computed with widths in $[0.4, 2.0]$m, with label projection to ground planes enabled.

\paragraph{Segmentation and tracking.}
We use FastSAM~\cite{Zhao23Arxiv-FastSam} compiled to a TensorRT engine at $480 \times 640$ resolution for open-vocabulary instance segmentation, 
with a minimum mask area of 400 pixels. Object tracking uses ByteTrack~\cite{Zhang22eccv-byteTrack} with appearance re-identification via a CLIP-based ReID model (TensorRT), 
a new-track threshold of 0.6, and a track buffer of 30 frames. A track must accumulate at least 8 observations before it qualifies for semantic annotation.

\paragraph{Semantic annotation.}
We use DAM-3B~\cite{Lian25iccv-DAM} as the vision-language model for object captioning, operating in focal-prompt mode with batch sizes of 32--128 masks. 
Frame selection uses Perception Encoder (PE-Core-L14-336) CLIP features to select the most visually informative frame per track. Per-frame CLIP features 
(ViT-B/16) are computed every 5 frames for temporal scene understanding. Semantic annotation queries are triggered every 120 frames.

\paragraph{Cross-view UQ.}
We use the cross-view uncertainty quantification mode, where the assignment BLP~\cite{Gorlo26cvpr-DAAAM} assigns $r{=}3$ representative 
frames per tracked fragment (with a minimum of 2 views required). The BLP is solved with GLPK\_MI (15s timeout), 
with a greedy set-cover fallback. Each assigned view produces an independent VLM caption; the log-determinant Gram uncertainty is computed from the $r$ 
L2-normalised caption embeddings with $\lambda = 10^{-8}$.
 We additionally incorporate the UMPIRE~\cite{Lau25icmlws-uqMllm} coherence term with weight $\alpha = 1.0$, scaling each embedding by its incoherence 
 weight $\gamma_i = \alpha (1 - \exp(\ell\ell_i))$ where $\ell\ell_i$ is the per-sequence log-likelihood.

\paragraph{Active view selection.}
Objects whose initial UQ score exceeds the threshold $\tau = -0.25$ are selected for refinement. Each flagged object receives up 
to $b{=}2$ additional views, chosen to maximally reduce the log-determinant via the rank-1 Schur complement update: 
$\log\det(\mathbf{G}^{(t+1)}) = \log\det(\mathbf{G}^{(t)}) + \log(c - \mathbf{b}^\top (\mathbf{G}^{(t)})^{-1} \mathbf{b})$, 
where $\mathbf{b}$ and $c$ are the cross-terms and self-similarity of the new view embedding. The Gram inverse is updated via the Sherman--Morrison formula, 
avoiding recomputation. Refinement runs asynchronously during pipeline shutdown with a timeout of 300s.

\paragraph{Downstream evaluation.}
For the NaVQA benchmark~\cite{Anwar25icra-remembr}, we use GPT-5-mini for question answering over the constructed scene graph memory, 
with sentence embeddings computed by Sentence-T5-XL~\cite{Ni22aclf-sentenceT5}. We evaluate on 7 CODa sequences (0, 3, 4, 6, 16, 21, 22) 
totalling 98,866 frames. All experiments are conducted on a single NVIDIA RTX 4090 GPU with 24GB VRAM, using 1 assignment worker and 1 grounding worker.
\begin{table}[t]
\centering
\begin{tabular}{lccc}
\toprule
Method & Accuracy $\uparrow$ & Position [m] $\downarrow$ & Temporal [min] $\downarrow$ \\
\midrule
{DAAAM-RR}  & $0.716 \pm 0.014$ & $41.18 \pm 2.31$ & $1.769 \pm 0.089$ \\
{DAAAM-QR}  & $0.738 \pm 0.011$ & $40.14 \pm 1.98$ & $1.734 \pm 0.072$ \\
{DAAAM-NC}~\cite{Khan24cvpr-consistency}  & $0.713 \pm 0.016$ & $41.62 \pm 2.54$ & $1.783 \pm 0.094$ \\
{DAAAM-LN}~\cite{Malinin20arxiv-uncertaintyAR}  & $0.714 \pm 0.015$ & $41.55 \pm 2.48$ & $1.776 \pm 0.091$ \\
{DAAAM-Ent}~\cite{Kuhn23arxiv-semanticUnc}  & $0.733 \pm 0.013$ & $40.33 \pm 2.17$ & $1.718 \pm 0.083$ \\
{DAAAM-ES}~\cite{Chen24arxiv-inside}  & $0.743 \pm 0.012$ & $39.69 \pm 2.02$ & $1.683 \pm 0.078$ \\
\midrule
{UQ-DAAAM} (Ours)  & $\mathbf{0.761} \pm 0.009$ & $\mathbf{37.84} \pm 1.74$ & $\mathbf{1.589} \pm 0.063$ \\
\bottomrule
\end{tabular}
\caption{Mean $\pm$ standard deviation on OC-NaVQA (All split). Results averaged over three seeds.}
\label{tab:results_std}
\end{table}

\section{Memory and Runtime Details}
\label{app:runtime}
We evaluate the overhead introduced by cross-view uncertainty quantification and active view selection on the same CODa 
sequences, using a single NVIDIA RTX 4090 GPU. Across 7 sequences (98,866 frames total), UQ-DAAAM achieves a mean frame 
rate of 13.5 Hz, comfortably exceeding the 10 Hz sensor rate of the CODa dataset and confirming that the UQ extensions 
preserve real-time operation. The cross-view BLP solver, which assigns $r{=}3$ views per tracked fragment, adds a mean 
latency of $131.7 \pm 132.5$ ms per query batch, accounting for only 1.49\% of total frame processing time. 
This is because the BLP runs on CPU via \texttt{cvxpy} and executes infrequently (828 calls across 98,866 frames). 
Each object receives $3.44 \pm 0.68$ cross-view captions on average, from which the log-determinant Gram uncertainty 
is computed via \texttt{numpy.linalg.slogdet} entirely on CPU with negligible overhead. Of the 18,298 objects 
annotated across all sequences, 6,324 (34.6\%) were selected for active view refinement based on their initial 
uncertainty, receiving $1.32 \pm 0.47$ additional views each. Active view selection reduced the mean UQ score from 
$0.56$ to $0.32$, indicating substantially increased caption agreement. The additional VLM calls for 
refinement occur asynchronously in the grounding worker thread during shutdown, with a configurable timeout. 
GPU memory consumption remains within the 24 GB envelope of the RTX 4090, as the UQ and BLP components are CPU-only 
and the VLM model (DAM-3B) is shared between initial annotation and refinement.

\section{Qualitative Results and Visualizations}
In the following qualitative examples, we present the captions and uncertainty scores for both initial and refined views of selected objects.
We show both objects with low uncertainty, which do not require refinement (before and after view selection are the same), and objects with high uncertainty, where the refined views lead to more consistent captions and reduced uncertainty.
In each visualization, we show the raw RGB images, object masks, and zoomed-in views of the masked objects for clarity.
\label{app:qual}
\subsection{Examples of Objects with Low Uncertainty}
\begin{figure}[H]
\centering
\includegraphics[width=0.7\textwidth]{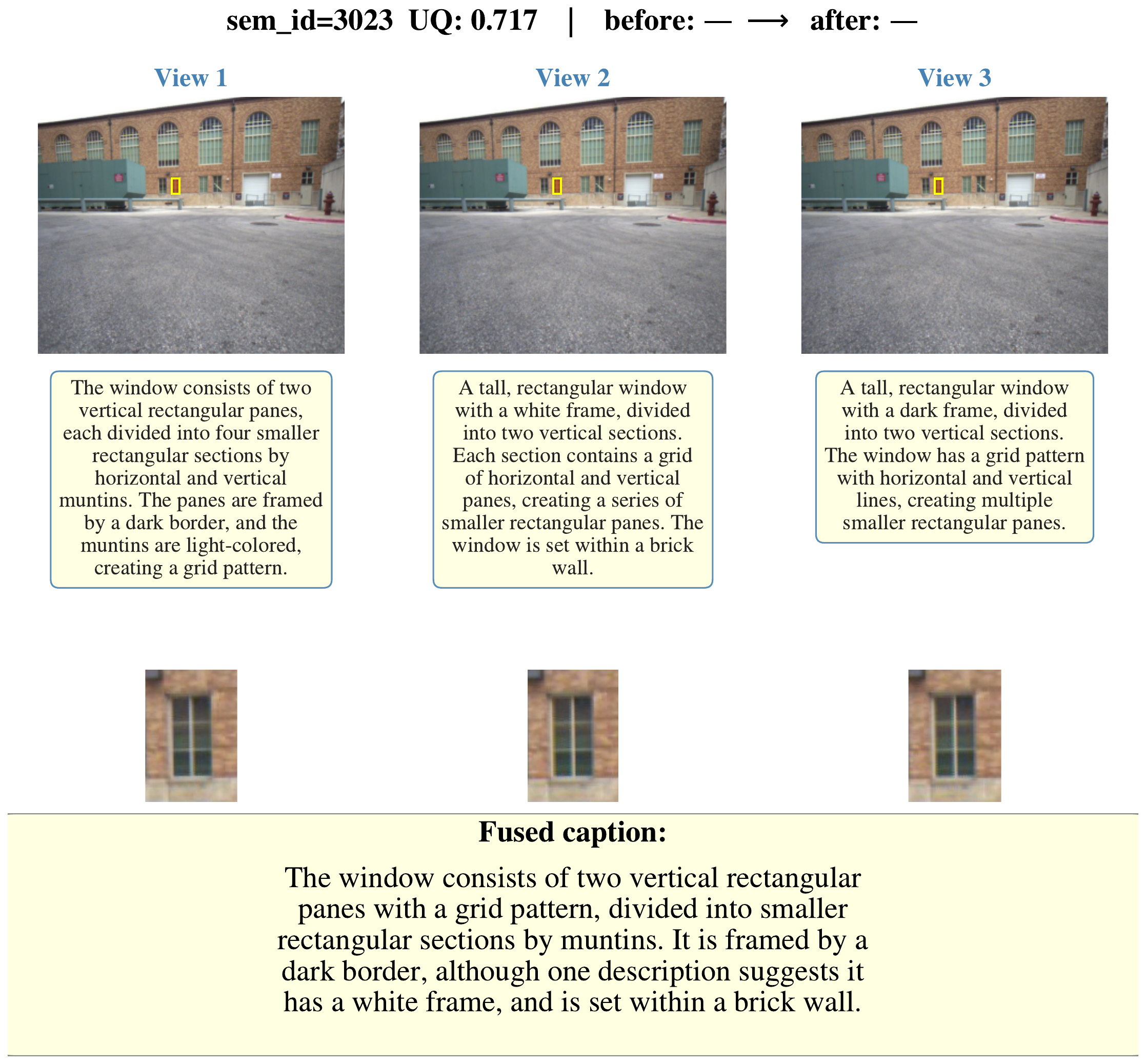}
\end{figure}
\begin{figure}[H]
\centering
\includegraphics[width=0.7\textwidth]{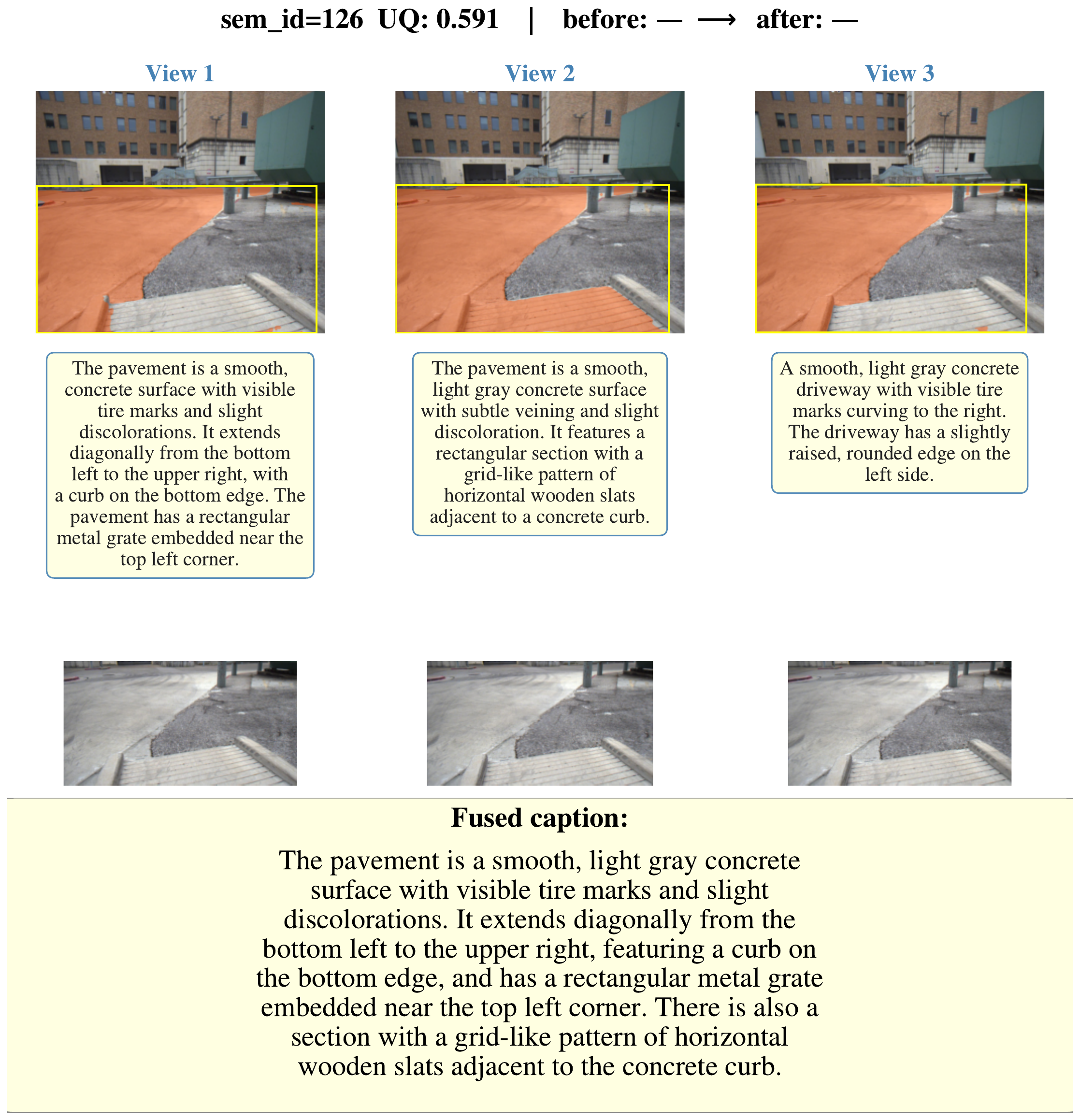}
\end{figure}
\begin{figure}[H]
\centering
\includegraphics[width=0.7\textwidth]{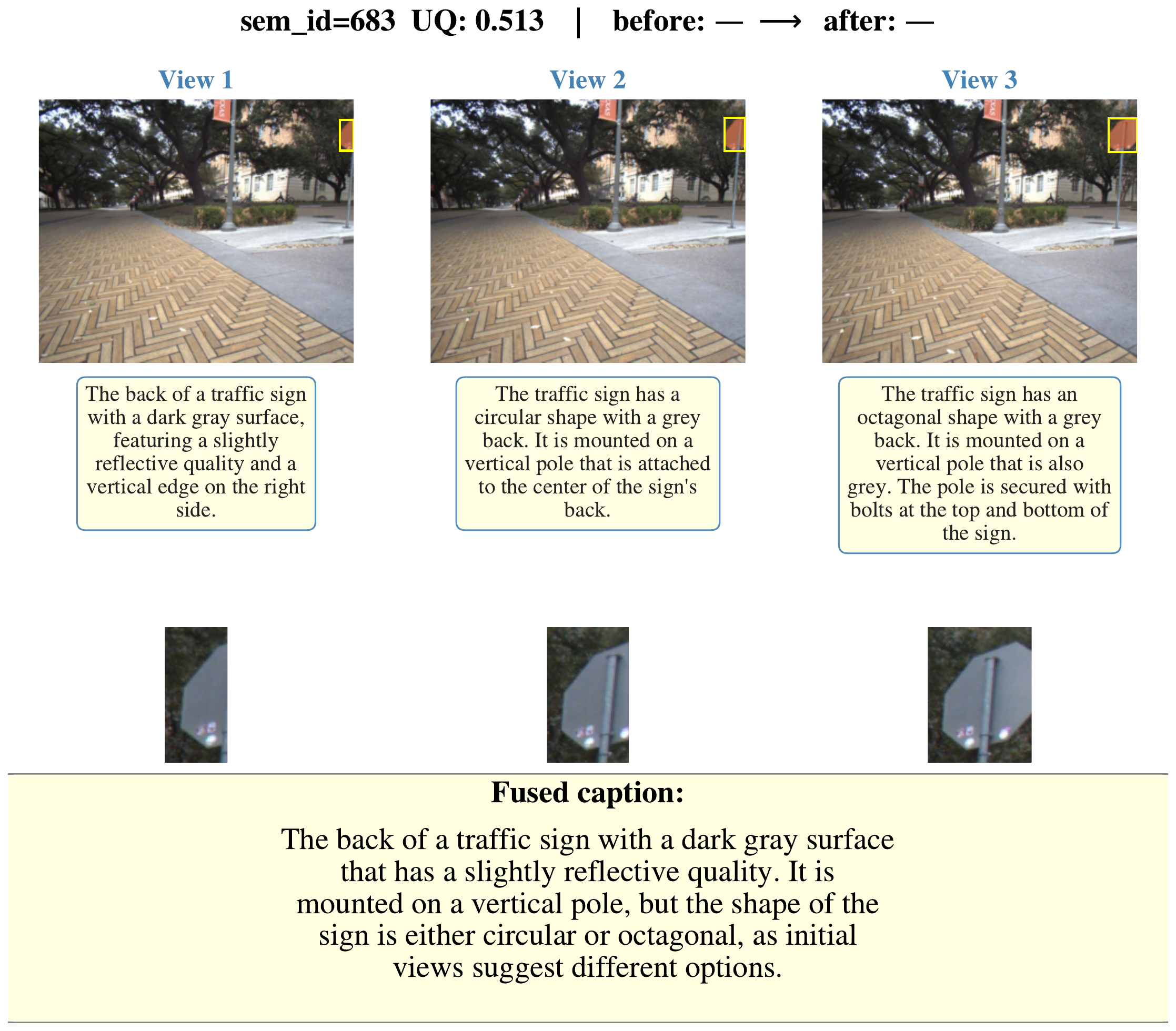}
\end{figure}
No refinement views are needed for these objects, since the initial views already provide a stable and coherent semantic summary. 
The captions are consistent across views, and the uncertainty scores are low, indicating that the system is confident in its understanding of these objects.
\subsection{Examples of Objects with High Uncertainty}
\begin{figure}[H]
\centering
\includegraphics[width=\textwidth]{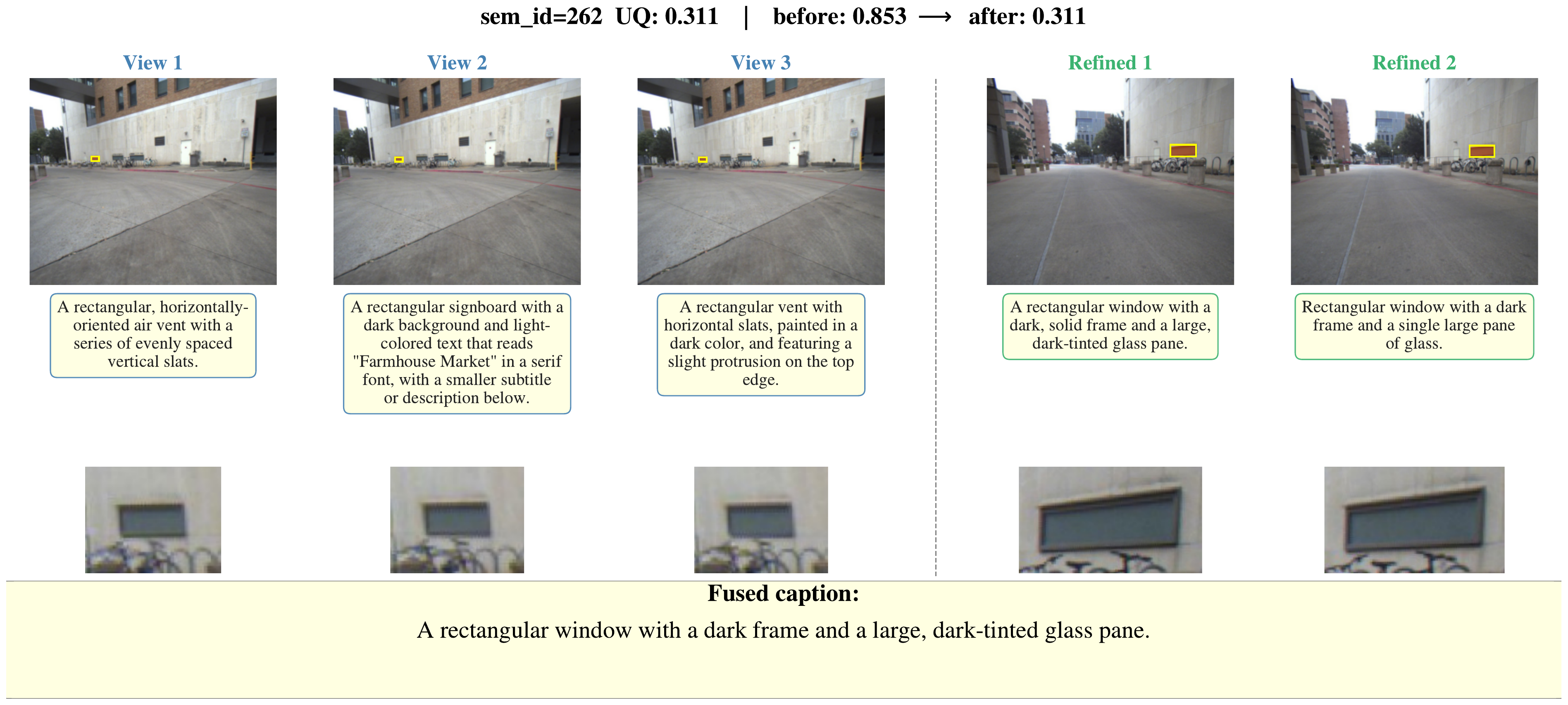}
\end{figure}
\begin{figure}[H]
\centering
\includegraphics[width=\textwidth]{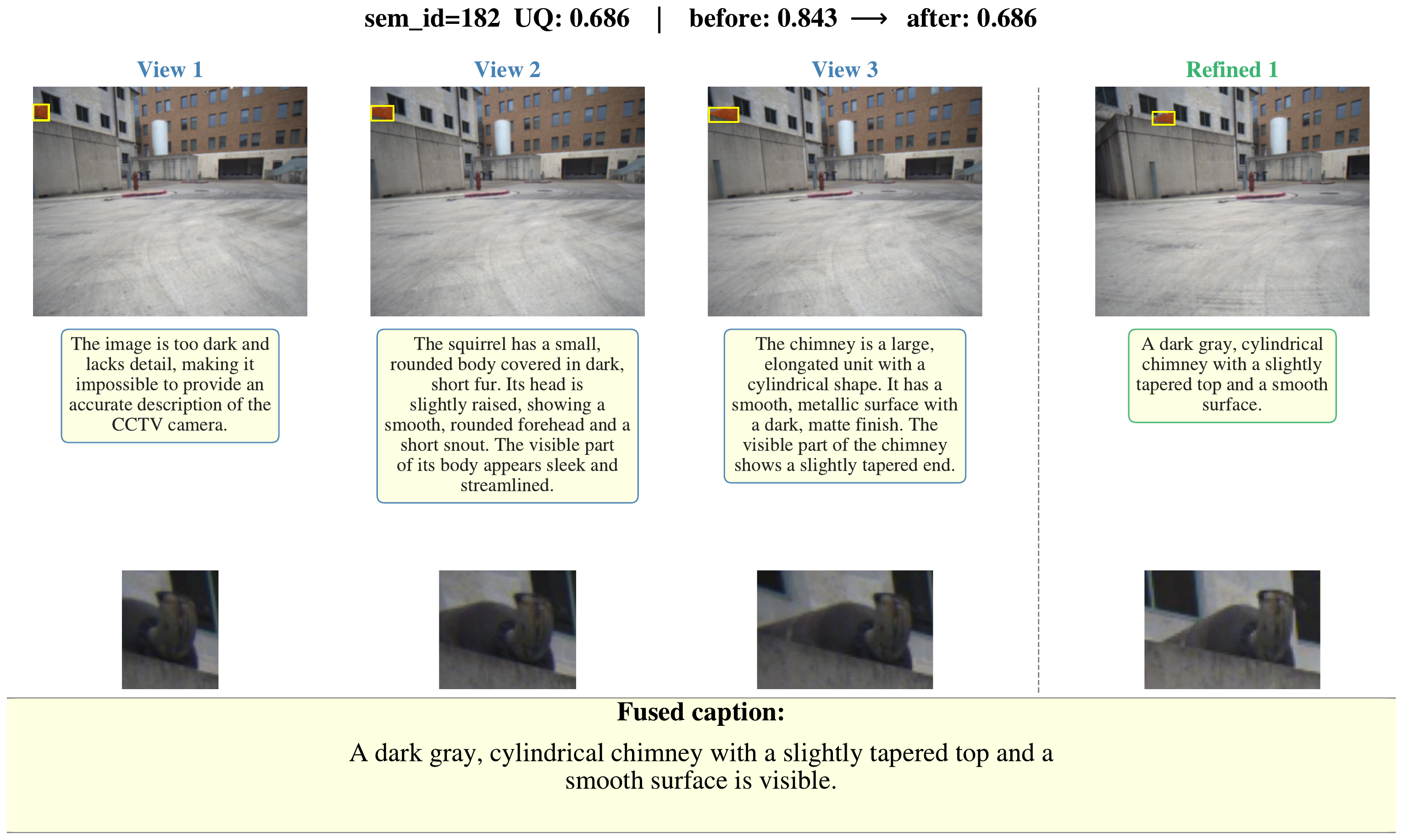}
\end{figure}
\begin{figure}[H]
\centering
\includegraphics[width=\textwidth]{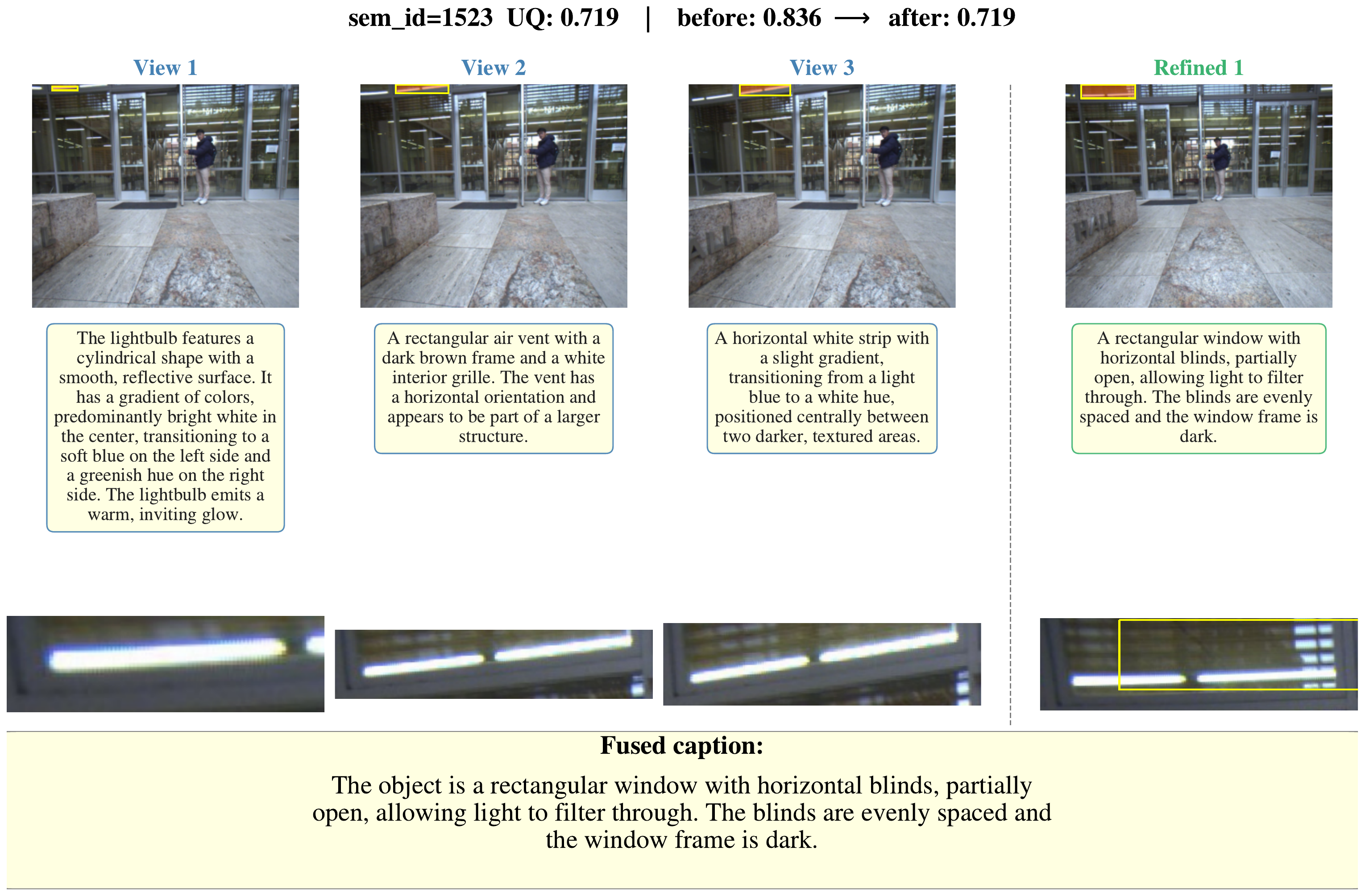}\\[1em]
\includegraphics[width=\textwidth]{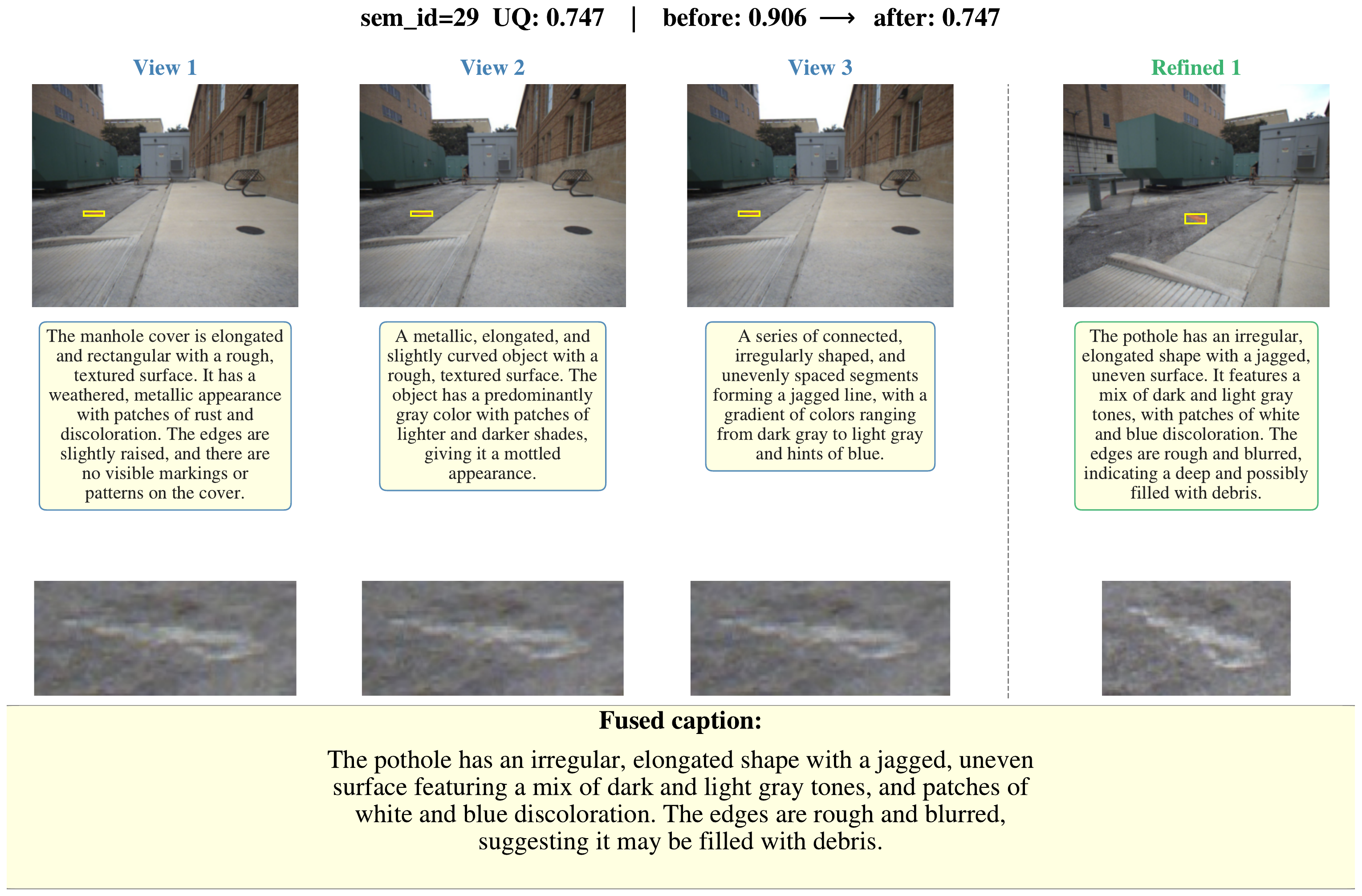}
\end{figure}
\begin{figure}[H]
\centering
\includegraphics[width=\textwidth]{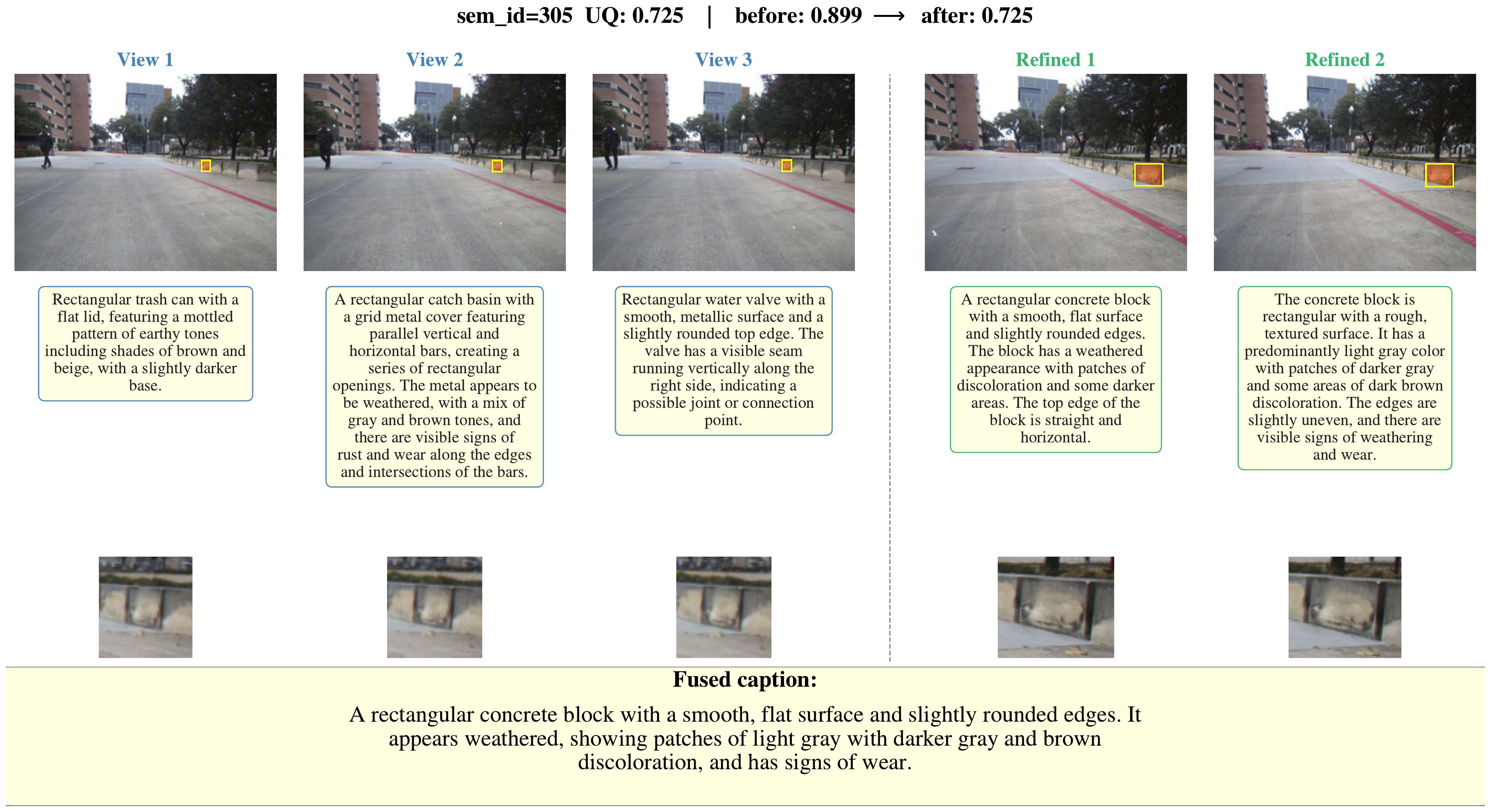}\\[1em]
\includegraphics[width=\textwidth]{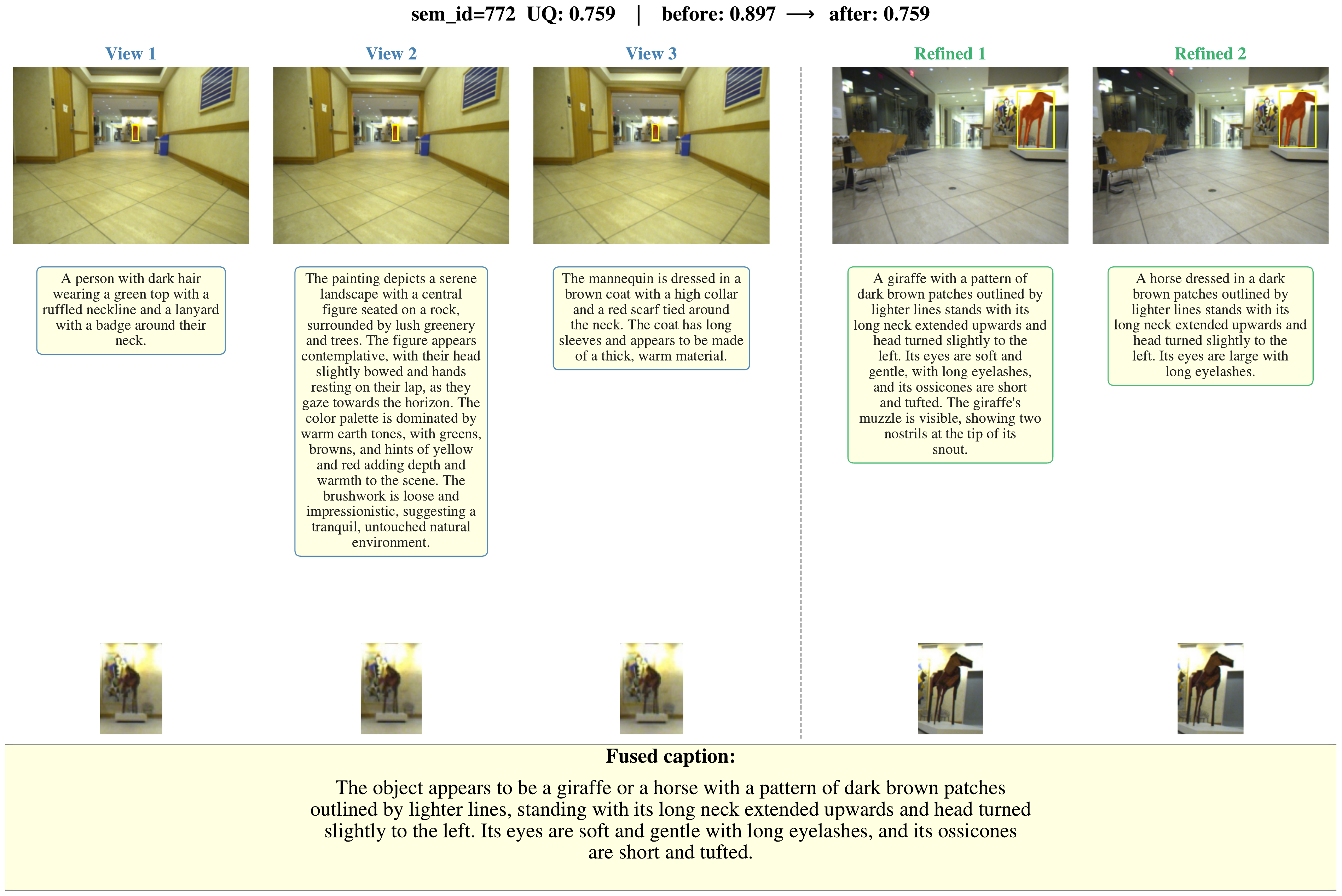}
\end{figure}
\begin{figure}[H]
\centering
\includegraphics[width=\textwidth]{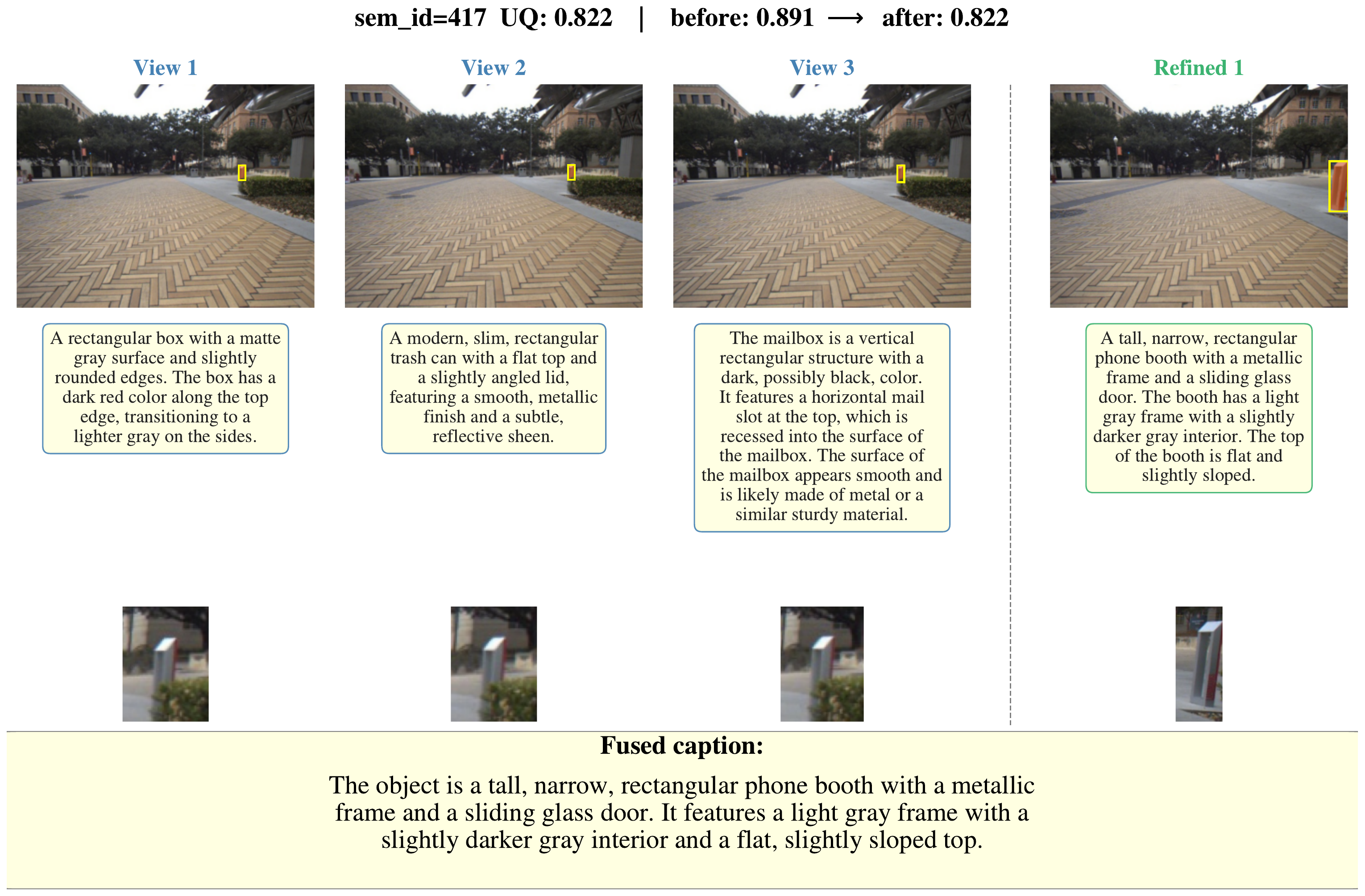}\\[1em]
\includegraphics[width=\textwidth]{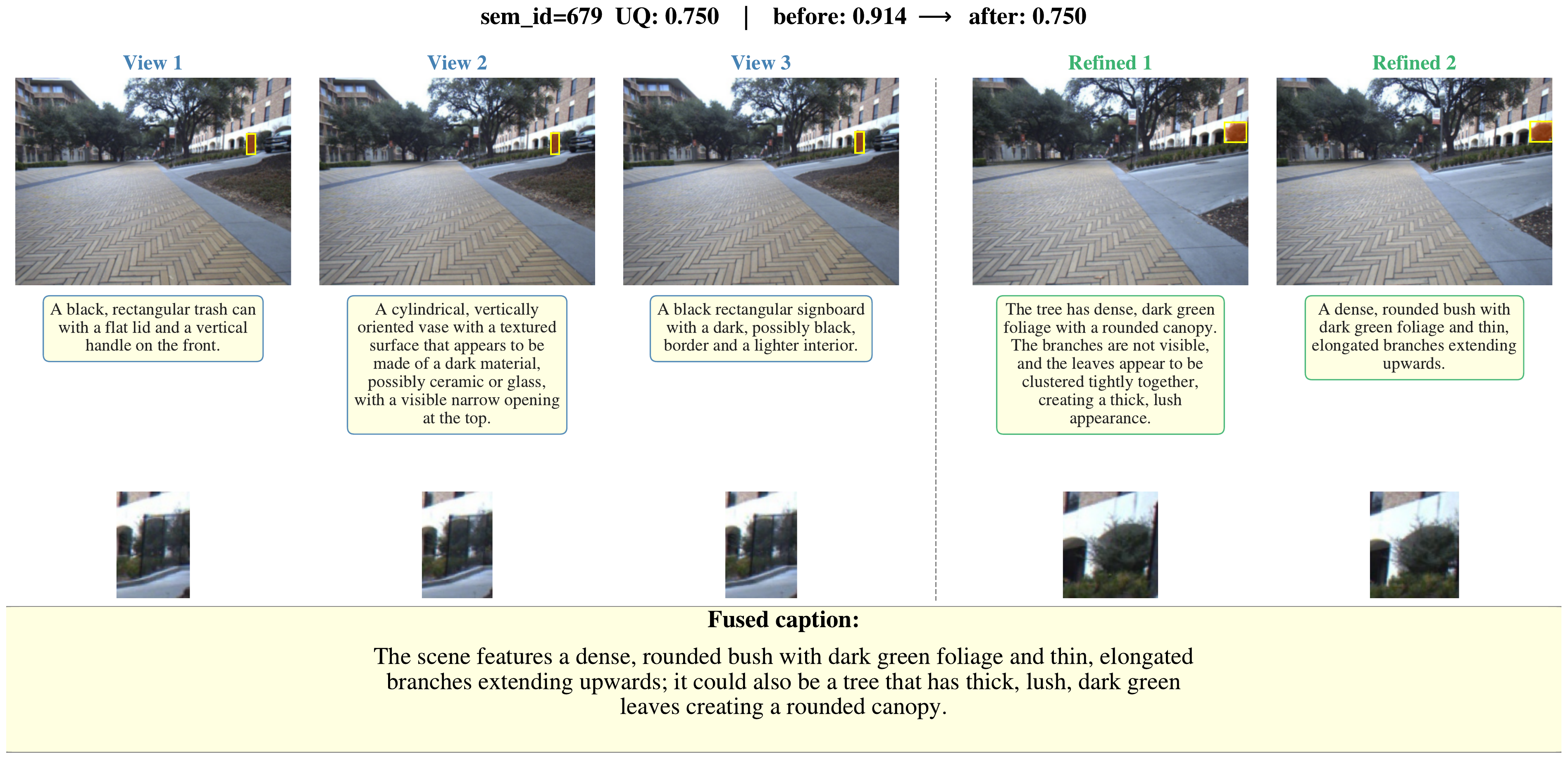}
\end{figure}
Across these qualitative examples, UQ-DAAAM consistently selects refined views that expose object evidence that is
 either missing or misleading in the initial views, leading to substantial reductions in uncertainty.
 In the first example (sem\_id=262), the initial views yield conflicting descriptions such as an air vent, a ``Farmhouse Market'' signboard, and a slatted structure,
 while the refined views, captured from a closer and more frontal viewpoint, reveal the object clearly as a dark-framed window with a tinted glass pane, reducing uncertainty from 0.853 to 0.311.
 In the second and third examples (sem\_id=182 and sem\_id=1523), the original captions are highly unstable and include implausible object hypotheses such as
 a CCTV camera, a squirrel, a glowing lightbulb, and an air vent; the refined views instead provide more frontal and informative observations
 that reveal a dark gray cylindrical chimney and a rectangular window with partially open horizontal blinds, lowering uncertainty from 0.843 to 0.686 and from 0.836 to 0.719, respectively.
 The fourth example (sem\_id=29) shows a harder case: the initial views describe the object as a manhole cover, a curved metallic object, or an irregular jagged surface pattern, while the refined view provides a sharper observation of
 a pothole-like structure, producing a moderate but still meaningful reduction from 0.906 to 0.747.
 The fifth example (sem\_id=305) shows that unstable initial hypotheses of a trash can, a grid metal catch basin, and a water valve are resolved by the refined views
 into a coherent description of a weathered rectangular concrete block, reducing uncertainty from 0.899 to 0.725.
 The sixth example (sem\_id=772) illustrates how distance-induced ambiguity, where the initial small crops are described as a person with a lanyard,
 an impressionistic landscape painting, and a mannequin in a brown coat, is mitigated by refined views that zoom in on the object and reveal an animal sculpture;
 here uncertainty reduction does not force a single interpretation, as the model remains unsure whether it is a horse or a giraffe, and the fused caption conservatively preserves both possibilities (0.897 to 0.759).
 The seventh example (sem\_id=417) highlights an important failure mode: although the refined view clearly improves the visual clarity of the object, which is actually a
 bike parking rack, the underlying VLM still misidentifies it as a tall phone booth with a sliding glass door, so the uncertainty only decreases modestly from 0.891 to 0.822.
 This case shows that refinement can expose a sharper observation while still being bottlenecked by the backbone VLM's recognition ability, and that lower uncertainty does not automatically imply semantic correctness.
 Finally, the eighth example (sem\_id=679) demonstrates that refinement could avoid occlusion: initial captions conflict between a trash can, a ceramic vase, and a signboard due to the occlusion of the metal fences,
 while the refined views narrow the ambiguity between a dense bush and a tree canopy and reduce uncertainty from 0.914 to 0.750, with the fused caption conservatively preserving both possibilities.
 Overall, these examples show that UQ-DAAAM tends to select views with better geometry, scale, or visibility, which makes the resulting captions more semantically consistent and
 leads to measurable uncertainty reduction, while also exposing residual failures that are attributable to the underlying VLM rather than to the uncertainty-driven view selection itself.

\section{Limitations}
\label{app:limitations}
UQ-DAAAM inherits several limitations from the underlying DAAAM pipeline while introducing additional ones of its own. 
First, our uncertainty estimate depends on the quality of the generated object captions and their associated coherence signals. 
If the VLM systematically hallucinates, misses fine-grained attributes, or produces poorly calibrated confidence estimates, 
the resulting uncertainty score may not faithfully reflect the true semantic ambiguity of the object. 
In particular, cross-view agreement does not always imply correctness, since multiple views may consistently support the 
same mistaken description. Second, the refinement stage introduces additional semantic-query cost. Although our method is 
designed to allocate this budget selectively, repeated querying and fusion may still be too expensive for highly dynamic 
platforms or settings with strict latency constraints. Finally, while caption fusion produces a more stable 
object summary, maintaining multi-view descriptions, uncertainty estimates, and refinement histories for many objects may 
increase memory overhead over long horizons. Future work should investigate 
stronger active-view acquisition strategies and memory-bounded summarization mechanisms for long-term deployment.
\clearpage

\end{document}